\documentclass[10pt]{article}

\usepackage{graphicx}
\usepackage{multirow}
\usepackage{xcolor}         
\usepackage[labelfont=bf,textfont=it,font=footnotesize]{caption}
\usepackage{parskip}
\usepackage{times}
\usepackage{booktabs}
\usepackage{listings}
\usepackage{float}
\usepackage{subcaption}
\usepackage{titlesec}

\usepackage{amssymb}
\usepackage{amsthm}
\usepackage{amsmath}
\usepackage{nicefrac}        
\usepackage{siunitx}

\usepackage{algorithm}
\usepackage{algpseudocode}
\usepackage[numbers,sort&compress]{natbib} 
\usepackage[utf8]{inputenc} 
\usepackage[T1]{fontenc}    
\usepackage{url}            
\usepackage{microtype}      
\usepackage{enumitem}
\usepackage{tikz}
\usepackage{pgfplots}
\usepackage{pgfplotstable}
\usepackage{multirow}
\pgfplotsset{compat=newest,scaled y ticks=true} 

\usetikzlibrary{spy}

\usepackage{tabulary}
\usepackage{wrapfig}
\usepackage[margin=1in]{geometry}

\graphicspath{{Figures/}}
\DeclareGraphicsExtensions{.pdf,.png,.jpg}

\usepackage{tabulary}
\usepackage{multirow}
\usepackage{enumitem}
\usepackage{siunitx}
\usepackage{tikz} 
\usetikzlibrary{backgrounds}
\usetikzlibrary{shapes,arrows,positioning,intersections,quotes}
\usepackage{pgfplots}
\usetikzlibrary{matrix}
\usetikzlibrary{positioning}
\usetikzlibrary{backgrounds}
\pgfplotsset{compat=newest,scaled y ticks=true} 
\usepgfplotslibrary{units} 
\usepackage{csvsimple}

\usepackage{url}

\usepackage[switch,pagewise]{lineno} 
\modulolinenumbers[5]

\usepackage[pdfborder={0 0 0},colorlinks,allcolors=DarkBlue]{hyperref}
\hypersetup{
    colorlinks,
    linkcolor={blue!80!black},
    citecolor={blue!80!black},
    urlcolor={blue!80!black},
    plainpages=true
}
\newcommand{\cref}[2]{\hyperref[#2]{#1~\ref*{#2}}}
\newcommand{\colref}[3]{\hyperref[#2]{#1~\ref*{#2}{#3}}}
\newcommand{\figref}[1]{\cref{Figure}{#1}}
\newcommand{\figrefA}[2]{\colref{Figure}{#1}{#2}}
\newcommand{\secref}[1]{\cref{Section}{#1}}
\newcommand{\eqnref}[1]{\cref{Equation}{#1}}
\newcommand{\tabref}[1]{\cref{Table}{#1}}
\newcommand{\algoref}[1]{\cref{Algorithm}{#1}}
\newcommand{\thmref}[1]{\cref{Theorem}{#1}}
\newcommand{\lemref}[1]{\cref{Lemma}{#1}}
\newcommand{\remref}[1]{\cref{Remark}{#1}}
\newcommand{\email}[1]{\href{mailto:#1}{\nolinkurl{#1}}}

\titlespacing\section{0pt}{6pt plus 1pt minus 0pt}{3pt plus 1pt minus 0pt}
\titlespacing\subsection{0pt}{4pt plus 1pt minus 0pt}{3pt plus 1pt minus 0pt}
\titlespacing\subsubsection{0pt}{4pt plus 1pt minus 1pt}{3pt plus 1pt minus 1pt}
\titleformat{\section}{\large\bfseries\sffamily}{\thesection}{1em}{}
\titleformat{\subsection}{\normalsize\bfseries\sffamily}{\thesubsection}{1em}{}
\titleformat{\subsubsection}{\normalsize\bfseries\sffamily}{\thesubsubsection}{1em}{}

\DeclareGraphicsExtensions{.pdf,.png,.jpg}
\graphicspath{{figure_panels/}}

\DeclareMathOperator*{\argmin}{arg\,min}
\theoremstyle{plain}

\newtheorem{theorem}{\textbf{Theorem}}[section]
\newtheorem{corollary}{Corollary}
\newtheorem{lemma}[theorem]{\textbf{Lemma}}
\newtheorem{assumption}[theorem]{\textbf{Assumption}}

\theoremstyle{remark}
\newtheorem{remark}{\textbf{Remark}}

\newcommand{\spatialD}{D}

\newcommand{\xcomp}[1]{x_{#1}}
\newcommand{\xii}{\xcomp{i}}

\newcommand{\xvec}{\mvec{x}}

\newcommand{\partialder}[2]{\frac{\partial #1}{\partial #2}}

\newcommand{\diffusivity}{\nu}

\newcommand{\pder}[1]{\partialder{p}{\xii}}

\newcommand{\nel}{n_{el}}
\newcommand{\uh}{u^h}

\newcommand{\spaceV}{V}
\newcommand{\spaceVh}{V^h}
\newcommand{\spaceVtheta}{V_\netparams}
\newcommand{\spaceVhtheta}{V^h_\netparams}

\newcommand{\mesh}{\mathcal{K}^h}

\newcommand{\grad}{\mvec{{\nabla}}}

\newcommand{\laplacian}{\Delta}

\newcommand{\uNeural}{u_{\netparams}}

\newcommand{\uExact}{u}
\newcommand{\uNeuralExactOpt}{u_{\netparamsExact}}
\newcommand{\uDiscrete}{\uh}

\newcommand{\eGen}{e_G}
\newcommand{\eOpt}{e_{\netparams}}
\newcommand{\eCap}{e_{\mathcal{H}}}
\newcommand{\eDis}{e_h}
\newcommand{\errJ}{\mathcal{E}_J}
\newcommand{\errL}{\mathcal{E}_{\hat{L}}}

\newcommand{\mvec}[1]{{\mathbf{#1}}}

\newcommand{\setbuilder}[1]{\left\{ #1 \right\}}

\newcommand{\norm}[1]{\| #1 \|}
\newcommand{\normL}[3][2]{\| #2 \|_{L^{#1}(#3)}}

\newcommand{\normV}[1]{\| #1 \|_{\spaceVh}}

\newcommand{\normVtheta}[1]{\| #1 \|_{\spaceVtheta}}
\newcommand{\normVhtheta}[1]{\| #1 \|_{\spaceVhtheta}}

\newcommand{\half}{\frac{1}{2}}
\newcommand{\oneOver}[1]{\frac{1}{#1}}

\newcommand{\neuralMap}{G_{nn}}
\newcommand{\netparams}{\theta}
\newcommand{\netparamsExact}{\tilde{\theta}}

\newcommand{\neufenet}{\textsc{NeuFENet}}

\def\gH{{\mathcal{H}}}

\newcommand{\logLogSlopeTriangle}[5]
{
	
	\pgfplotsextra
	{
		\pgfkeysgetvalue{/pgfplots/xmin}{\xmin}
		\pgfkeysgetvalue{/pgfplots/xmax}{\xmax}
		\pgfkeysgetvalue{/pgfplots/ymin}{\ymin}
		\pgfkeysgetvalue{/pgfplots/ymax}{\ymax}
		
		\pgfmathsetmacro{\xArel}{#1}
		\pgfmathsetmacro{\yArel}{#3}
		\pgfmathsetmacro{\xBrel}{#1-#2}
		\pgfmathsetmacro{\yBrel}{\yArel}
		\pgfmathsetmacro{\xCrel}{\xArel}
		
		\pgfmathsetmacro{\lnxB}{\xmin*(1-(#1-#2))+\xmax*(#1-#2)} 
		\pgfmathsetmacro{\lnxA}{\xmin*(1-#1)+\xmax*#1} 
		\pgfmathsetmacro{\lnyA}{\ymin*(1-#3)+\ymax*#3} 
		\pgfmathsetmacro{\lnyC}{\lnyA+#4*(\lnxA-\lnxB)}
		\pgfmathsetmacro{\yCrel}{\lnyC-\ymin)/(\ymax-\ymin)} 
		
		\coordinate (A) at (rel axis cs:\xArel,\yArel);
		\coordinate (B) at (rel axis cs:\xBrel,\yBrel);
		\coordinate (C) at (rel axis cs:\xCrel,\yCrel);
		
		\draw[#5]   (A)-- node[pos=0.5,anchor=north] {1}
		(B)-- 
		(C)-- node[pos=0.5,anchor=west] {#4}
		cycle;
	}
}

\newbox\dumbox
\newcommand{\mymark}[2]{%
	\setbox\dumbox=\hbox{#2}%
	\hbox to \wd\dumbox{\hss%
		\tikz[overlay,remember picture,baseline=(#1.base)]{ \node (#1) {\box\dumbox}; }%
		\hss}%
}

\definecolor{cpu3}{HTML}{F44336}
\definecolor{cpu4}{HTML}{2196F3}
\definecolor{cpu1}{HTML}{4CAF50}
\definecolor{cpu2}{HTML}{FFC107}
\definecolor{gpu3}{HTML}{EF9A9A}
\definecolor{gpu4}{HTML}{90CAF9}
\definecolor{gpu1}{HTML}{A5D6A7}
\definecolor{gpu2}{HTML}{FFE082}

\definecolor{cpu5}{HTML}{9932CC}

\newcommand\numRowsK{3}
\newcommand\numColsK{3}
\newcommand{\K}[2]{
	\edef\Kcol##1##2##3{###2}%
	\edef\Krow##1##2##3{\noexpand\Kcol###1}%
	\Krow
	{1 0 1}
	{0 1 0}
	{1 0 1}%
}

\begin{document}

\begin{center}
{\usefont{OT1}{phv}{b}{sc}\selectfont\Large{\neufenet{}: Neural Finite Element Solutions with Theoretical Bounds for Parametric PDEs}}

{\usefont{OT1}{phv}{}{}\selectfont\normalsize
{Biswajit Khara$^1$, Aditya Balu$^1$, Ameya Joshi$^2$, Soumik Sarkar$^1$, Chinmay Hegde$^2$, Adarsh Krishnamurthy$^1$, Baskar Ganapathysubramanian$^1$*}}

{\usefont{OT1}{phv}{}{}\selectfont\normalsize
{$^1$ Iowa State University\\
$^2$ New York University\\
* Corresponding author: \email{barkarg@iastate.edu}\\
}}
\end{center}


\section*{Abstract}
We consider a mesh-based approach for training a neural network to produce field predictions of solutions to parametric partial differential equations (PDEs). This approach contrasts current approaches for ``neural PDE solvers'' that employ collocation-based methods to make point-wise predictions of solutions to PDEs. This approach has the  advantage of naturally enforcing different boundary conditions as well as ease of invoking well-developed PDE theory--including analysis of numerical stability and convergence-- to obtain capacity bounds for our proposed neural networks in discretized domains. We explore our mesh-based strategy, called \neufenet{}, using a weighted Galerkin loss function based on the Finite Element Method (FEM) on a parametric elliptic PDE. The weighted Galerkin loss (FEM loss) is similar to an energy functional that produces improved solutions, satisfies \textit{a priori} mesh convergence, and can model Dirichlet and Neumann boundary conditions. We prove theoretically, and illustrate with experiments, convergence results analogous to mesh convergence analysis deployed in finite element solutions to PDEs. These results suggest that a mesh-based neural network approach serves as a promising approach for solving parametric PDEs with theoretical bounds.

\subsection*{Keywords}
Neural solvers |
Deep learning |
Physics informed learning |
Parametric PDE |
Data-free modeling

\section{Introduction}

Scientific machine learning is an emerging field that combines developments in machine learning with scientific computation methods. This field has witnessed a variety of approaches that deploy neural networks to solve partial differential equations (PDE). Such \emph{neural PDE solvers} provide a very different strategy of solving differential equations as compared to the traditional numerical methods; they primarily rely on \emph{optimization} techniques rather than the exact solution of systems of equations. This recent explosion in this line of work was initiated by the seminal paper on Physics Informed Neural Networks (PINNs)~\citep{raissi2019physics}.

In terms of the amount of data usage, neural PDE solvers span a wide spectrum. Some methods are ``data-driven''~\citep{rudy2019data,tompson2017accelerating,raissi2018hidden,lu2019deeponet}, where the solution to a given PDE is constructed from available experimental data or the underlying PDE is inferred from available data (commonly termed as the discovery of hidden physics). In contrast, at the other end of the spectrum are the so-called ``data-free'' methods that do not rely on input-output pairs but solely use the PDE and the boundary conditions to obtain the solution. In the past few years, many such methods have been proposed~\citep{ kharazmi2019variational, sirignano2018dgm,yang2018physics, pang2019fpinns, karumuri2020simulator, han2018solving, michoski2019solving, samaniego2020energy, ramabathiran2021spinn, lu2019deeponet,botelho2020deep}. Our work in this paper follows the latter data-free approach.

The core of neural PDE solvers are deep neural networks, which can represent arbitrarily complicated functions from the input to the output domain and, therefore, can approximate the PDE solution. Most neural methods use a pointwise prediction framework (also known as implicit neural networks~\citep{sitzmann2020implicit}). These pointwise prediction frameworks take $ \mvec{x} \in \spatialD$ (the spatial coordinates of the field) as input and produce an output solution value of $u(\mvec{x})$ (the solution field value at $\mvec{x}$) as shown in \figref{fig:pinn-mapping-sketch}. Thus, the neural PDE solvers create a mapping between the input domain $\spatialD$ to the range of the solution. Due to a pointwise prediction framework, these methods do not require a mesh and thus rely on collocating points from the domain. To take advantage of the modern stochastic gradient descent (SGD) based methods, this set of collocation points are often selected in a random or quasi-random manner~\citep{mishra2021enhancing}. The trained network approximates the discrete solution via a complicated and nonlinear mapping. This is in contrast to classical numerical PDE approaches which usually rely on a linear combination of local functions with limited differentiability (even when the exact solution may be analytic). However, such poinwise prediction neural methods do not naturally  account for the domain topology. In particular, the ``local'' nature of the solution and sparsity of matrices that emerges naturally in classical methods are missing in these neural methods. Some researchers have explored the idea of using classical methods such as finite difference methods (FDM) and finite volume methods (FVM) to construct neural architectures for solving PDEs~\citep{zhu2019physics,wen2021u,ranade2021discretizationnet}. Inspired by traditional numerical techniques, these frameworks construct a mapping between an input field and the solution field while using the discretization techniques associated with conventional numerical methods. These methods take advantage of the ``local'' nature of the solution and sparsity of matrices, similar to traditional numerical methods. In particular, mathematical concepts from finite element methods (FEM) are naturally translatable to neural networks (quadrature can be represented as convolutions), and provide interesting possibilities including variational arguments (monotone convergence to the solution), mesh convergence, basis order based convergence, and natural incorporation of boundary conditions. The current work builds upon these ideas. 

\noindent \textbf{Neural architecture:} In this paper, we develop a finite element (FEM) based neural architecture for solving PDEs. \figref{fig:diffnent-mapping-sketch} shows an abstract outline of this idea where the mapping is obtained with the use of convolutional neural networks, a specific class of network architectures specialized in learning from discrete domains such as input field $ S^d $ and the output field $ U^d $ as shown in \figref{fig:diffnent-mapping-sketch}. The nature of the input field $ S^d $ and the output field $ U^d $ (in \figref{fig:diffnent-mapping-sketch}) would depend on the actual PDE under consideration and will be made more concrete in later sections. To the best of the knowledge of the authors, no efforts have been made to develop FEM-based neural architectures. There are several benefits in developing a finite element method (FEM) based neural architecture. FEM based numerical methods are often backed by a well developed and elegant theory that connects the discretization of the domain (in terms of element/cell dimension, $h$) and the properties of the basis functions used to approximate the field (in terms of polynomial order, $\alpha$) with the quality of the ensuing numerical solution to the PDE. In particular, numerical stability arguments and \textit{a priori} error estimates allow users to judiciously reason about the accuracy, robustness, and convergence~\citep{brenner2007mathematical,larson2013finite}. Such theoretical arguments rely on the spatial discretization of the domain and properties of the basis functions in finite elements\footnote{In contrast, state-of-art neural methods allow us to use basis functions beyond polynomials/Fourier bases and approximate much more complicated mappings. Although such methods \emph{can} be analyzed theoretically, the estimates are often impractical~\citep{shin2020error, mishra2020estimates, jiao2021error}. This is a very active area of research, and we expect tighter estimates in the near future.}.

\noindent \textbf{Loss functions}: Having decided on borrowing the discretization scheme from FEM, there are multiple avenues to define the loss function. First, we need to define the spatial derivatives at each quadrature point. In FEM, this is done by directly differentiating the basis functions. Pointwise prediction methods perform this by differentiating the neural network with respect to the input variable. The differentiation process in the numerical method is straightforward and interpretable, while that is not necessarily the case in pointwise neural methods. Once the spatial derivatives are defined via the basis functions, we can either compare the weak form against the predefined basis functions from the test function space and perform a residual minimization or use an energy minimization approach. In the present work, we choose the latter, also known as the Rayleigh-Ritz (RR) method \cite{evans1998partial,reddy2010introduction}. The RR method states that the solution to a PDE must be the stationary point of some functional (i.e., ``energy'') under certain conditions. We note that the RR method has been used in a neural network for solving PDEs before \citep{yu2017deep}. However, the approach used there closely matches the pointwise prediction approach outlined in \figref{fig:pinn-mapping-sketch}, in contrast to our proposed approach.

\noindent \textbf{Boundary conditions}: The imposition of boundary conditions can also be challenging in neural methods. Very few neural methods satisfy/apply the boundary conditions exactly \citep{lee1990neural, lagaris1998artificial, malek2006numerical,sukumar2021exact}, with most methods relying on approximate approaches \citep{lagaris2000neural, raissi2019physics, sirignano2018dgm} usually by including an additional loss function corresponding to the imposed boundary conditions. It has been shown by \citet{van2020optimally} and \citet{wang2020understanding} that these losses have to be carefully weighed, making this a non-trivial exercise in hyperparameter tuning. This hyperparameter sensitivity underlines the difficulty of applying the boundary conditions in a neural network-based method (or simply neural method). Also, note that the method in \citep{yu2017deep} (using RR method) is unable to apply the Dirichlet boundary conditions precisely and consequently make use of a penalty-based approach for imposing a Dirichlet boundary condition, which we avoid altogether.

\noindent \textbf{Parametric PDEs}: Going beyond a single PDEs, there is growing interest in neural approaches that solve parametric PDEs (i.e., PDEs defined by a family of parametrized boundary conditions or coefficient fields). Most neural PDE methods have so far been limited to solving for a single instance of the PDE than a class of parametric solutions. Extending a instance PDE solver into a parametric PDE solver can greatly augment rapid design exploration, as alluded to in SimNet~\citep{hennigh2021nvidia} and \citet{wang2021long}, where the authors build a conventional implicit neural solver for parametric PDEs. 

\pagebreak

In this paper, we build upon recent efforts that train networks to predict the full-field solution~\citep{paganini2018calogan, botelho2020deep, zhu2019physics} on parametric PDEs. Our contributions are as follows:
\begin{enumerate}[topsep=0.0in,itemsep=0in]
	\item We present an algorithm that bridges traditional numerical methods with neural methods. The neural network is designed to map inputs to the discretized field solution $ u $. However, the neural network is \textit{not responsible} for ensuring the spatial differentiability of the solution. Rather, the discrete field solution relies on traditional numerical methods (and associated numerical differentiation and quadrature) to construct the loss function. Such an approach allows the natural incorporation of different boundary conditions and allows \textit{a priori} error estimates.
	\item We define the loss functions based on the Rayleigh-Ritz method coupled with the approximation scheme provided by a continuous Galerkin FEM. By defining such loss functions, we utilize function spaces with appropriate differentiability. This also account for the ``local'' nature of the solution resulting in computationally efficient loss evaluations. 
	\item We prove error convergence (similar to conventional mesh convergence) for a particular class of PDEs.
	\item We demonstrate \neufenet{}'s performance on linear Poisson equation in 2D and 3D with both Dirichlet and Neumann boundary conditions. Further, we test the parametric capability of this method on Poisson's equation, by considering a case involving stochastic diffusivity which requires access to a parameteric PDE solver.
\end{enumerate}

\begin{figure}[!t]
    \centering
	\begin{minipage}{.49\textwidth}
		\begin{tikzpicture}
		\draw (0, 0) node[inner sep=0] {\includegraphics[trim=0 0 0 0,clip,width=0.99\linewidth]{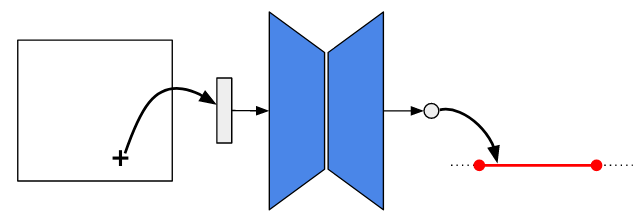}};
		\draw (-3.0, -0.3) node {$\mvec{x}_i$};
		\draw (1.3, 0.4) node {$u_{i}$};
		\draw (4.0, -0.4) node {$\mathbb{R}$};
		\draw (1.8, -1.15) node {$a$};
		\draw (3.3, -1.1) node {$b$};
		\draw (0.2, -1.8) node {$G_{nn}$};
		\end{tikzpicture}
		\subcaption{Conventional neural view: domain-to-range mapping}
		\label{fig:pinn-mapping-sketch}
	\end{minipage}	
	\hfill
	\begin{minipage}{.47\textwidth}
		\begin{tikzpicture}
		\draw (0, 0) node[inner sep=0] {\includegraphics[trim=0 0 0 0,clip,width=0.99\linewidth]{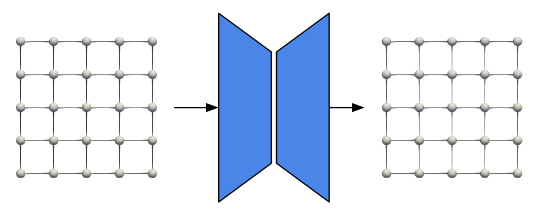}};
		\draw (-2.9, -1.8) node {$S^d$};
		\draw (2.9, -1.8) node {$U^d$};
		\draw (0.2, -1.8) node {$G_{nn}$};
		\end{tikzpicture}
		\subcaption{A more complex view: field-to-field mapping}
		\label{fig:diffnent-mapping-sketch}
	\end{minipage}	
	\caption{(a) Several neural methods (like \citep{lagaris1998artificial,raissi2019physics}) are trained to produce point predictions: $ \neuralMap : D\rightarrow \mathbb{R}$, which are easier to train, but more difficult to analyze and converge~\citep{krishnapriyan2021characterizing,wang2021understanding}, (b) In the \neufenet{} approach, we train a neural network to produce a discretized field solution over a mesh. Such an approach directly links powerful PDE analysis techniques at the cost of a larger network. The terms $ S^d $ and $ U^d $ can be considered the discrete version of any meaningful pair of input and output relevant to the PDE; these will be made precise in the next section.} 
	\label{fig:common-networks-for-pdesolve}
\end{figure}

The rest of the paper is arranged as follows: the definitions and terminologies regarding the parametric Poisson's equation are introduced in \secref{sec:preliminaries} and the mathematical formulations are described in detail in \secref{sec:formulations}. The implementation aspects of \neufenet{} are described in \secref{sec:implementations}. \secref{sec:errors} presents a theoretical analysis of the errors and finally computational results are presented in \secref{sec:results}.

\section{Mathematical Preliminaries}\label{sec:preliminaries}

Consider a bounded open (spatial) domain $D \in \mathbb{R}^n, n\geq 2$ with a Lipschitz continuous boundary $\Gamma = \partial D$. We will denote the domain variable as $\mvec{x}$, where the boldface denotes a vector or tuple of real numbers. In $ \mathbb{R}^n $, we have $ \mvec{x} = (x_1,x_2,\ldots,x_n) $; but for 2D and 3D domains, we will frequently use the notation $ \mvec{x} = (x,y) $ and $ \mvec{x} = (x,y,z) $ respectively. Consider also a probability space $ (\Omega, F, P) $, where $ \Omega $ is the sample space, $ F $ is the $ \sigma $-algebra of the subsets of $ \Omega $ and $ P $, a probability measure. We consider an abstract PDE on the function $u: D\times \Omega \rightarrow \mathbb{R}$ as:
\begin{subequations}\label{pde:abstract-representation-continuous}
	\begin{align}
	\mathcal{N}[u; s(\mvec{x},\omega)] &=f(\mvec{x}),\quad \mvec{x}\in D,\ \omega\in\Omega \label{pde:abstract-equation}\\
	\mathcal{B}[u]&=g(\mvec{x}),\quad \mvec{x}\in\Gamma. \label{pde:abstract-bc}
	\end{align}
\end{subequations}
Here, $\mathcal{N}$ is a differential operator (possibly nonlinear) operating on a function $u$. The differential equation depends on the input-data (e.g., material property) $s$ which in turn is a function of the domain variable $ \mvec{x} $ and parameter $ \omega \in \Omega$. Thus $ \mathcal{N} $ is essentially a family of PDE's parameterized by $ \omega $. \footnote{While a probability based definition of $\omega$ is not needed for defining a parameteric PDE, we choose this definition for two reasons. First, such a formulation allows easy extension to the stochastic PDE case. Second, such a formulation will allow using expectation based arguments in the analysis of convergence.} $ \mathcal{B} $ is a boundary operator on $ u $. In general, there can be multiple boundary operators for different part of the boundary $ \Gamma $.

Given a PDE along with some boundary conditions, such as the one presented in \eqnref{pde:abstract-representation-continuous}, the goal is to find a solution $ u $ that satisfies \eqnref{pde:abstract-representation-continuous} as accurately as possible. Previous works such as \citep{lagaris1998artificial,raissi2019physics,sirignano2018dgm} seek to find a pointwise mapping $u: D \rightarrow \mathbb{R}$. Here (see next section), by coupling deep neural networks with numerical methods, we explore other mappings to retrieve a discrete field solution.

In this work, we focus on the Poisson's equation considering both Dirichlet and Neumann boundary conditions, along with a heterogeneous and stochastic diffusivity:
\begin{align}\label{eq:poisson-pde}
	-\grad\cdot(\diffusivity(\mvec{x},\omega)\grad u) &= f(\mvec{x}) \text{  in  } D
\end{align}
along with the boundary conditions
\begin{align}
	u &= g \text{ on } {\Gamma_D}\label{eq:poisson-bc-dirichlet} \\
	\partialder{u}{n} &= h  \text{ on } {\Gamma_N} \label{eq:poisson-bc-neumann}
\end{align}
where $ \diffusivity $ is the \emph{permeability} (or \emph{diffusivity}) which depends on both $ \xvec $ and the random variable $ \omega $; and $ f $ is the forcing. In relation to \eqnref{pde:abstract-representation-continuous}, $ \diffusivity $ plays the role of the data $ s $. $ \Gamma_D $ and $ \Gamma_N $ are the boundaries of the domain $ D $ where Dirichlet and Neumann conditions are specified respectively. We will assume that $ \partial D = \Gamma = \Gamma_D \cup \Gamma_N$.

\subsection{Poisson's equation in heterogeneous media}\label{sec:poisson-kl-problem-intro}

We are mostly interested in the problem of a steady state mass (or heat) transfer through an inhomogeneous medium (material). This essentially means that the material has different properties at different points. The only material property appearing in the Poisson's equation (\eqnref{eq:poisson-pde}) is $ \nu(\mvec{x}) $, thus the inhomogeneity can be modeled by a spatially varying $ \nu $, i.e., $ \nu = \nu(\mvec{x}) $. The equation and the BC's are given by:
\begin{subequations}\label{eq:poisson-kl-intro}
	\begin{align}
		-\grad\cdot(\tilde{\nu}(\mvec{x})\grad u) &= 0 \text{  in  } D \\
		u(0, y) &= 1 \\
		u(1, y) &= 0 \\
		\partialder{u}{y}(x,0) &= 0 \\
		\partialder{u}{y}(x,1) &= 0
	\end{align}
\end{subequations}
where $ D $ is a hypercube domain in $ \mathbb{R}^n $, $ n = 2,3 $. The diffusivity/permeability $ \tilde{\diffusivity} $ is heterogeneous with respect to $ \xvec $ and is also parameterized by $ \omega\in\Omega $. The specific form of $ \tilde{\diffusivity} $ is given in \eqnref{def:poisson-nu-harmonic-expansion}. 

\section{Formulations}\label{sec:formulations}

\subsection{Neural approximation of the solution}\label{sec:neural-approximation-of-solution}

Instead of seeking a mapping between the domain and an interval on the real line (\figrefA{fig:common-networks-for-pdesolve}{(a)}), we seek a mapping between the input $ s $ and the full field solution $ u $ in the discrete spaces (\figrefA{fig:common-networks-for-pdesolve}{(b)}). $S^d$ denotes the discrete representation of the known quantity $s$. $S^d$ could be either available only at discrete points (perhaps from some experimental data); or in many cases, $s$ might be known in a functional form and thus $S^d$ will be simply the values of $s$  evaluated on the discrete points. Therefore, if we denote a \neufenet{} (see \figrefA{fig:common-networks-for-pdesolve}{(b)}) network by $ \neuralMap $, then $ \neuralMap $ takes as input a discrete or functional representation of $ s $ and outputs a discrete solution field $U^d_\theta$ as:
\begin{align}\label{def:neural-function-definition}
	U^d_\netparams = \neuralMap(S^d; \netparams),
\end{align}
where $\netparams$ denote the network parameters. The mathematical formulations presented in this section only assumes a suitable neural network that provides the aforementioned mapping between $ S^d $ and $ U^d $. The actual network architecture is discussed in \secref{sec:model-arch}.

An untrained network, as expected, will produce a mapping that does not satisfy the discrete PDE and will possess a large error. Our goal is to bring this error down to an acceptable level, and thereby reaching a solution that is ``close enough'' to the exact solution. We do this by designing the loss function based on major ideas taken from the classical numerical methods, as explained below.

\subsection{Loss functions inspired by variational methods}\label{sec:loss-functions}

The design of the loss function, along with the choice of the neural mapping, forms the central part of our approach. The finite element based loss function is inspired by the Galerkin formulation of an elliptic PDE as well as the Rayleigh-Ritz method. In this case, we actually construct a function of certain regularity in the domain variable $ \xvec $, as opposed to just assuming a function of certain differentiability at the collocation points.

Suppose $ H^1(D) = W^{1,2}(D) $ denotes the Sobolev space of functions whose first derivatives are square integrable. Define the space $ V $ as
\begin{align}\label{eq:def:function-space-V}
	V = \setbuilder{v\in H^1(D) : v(0,y) = 1, v(1,y) = 0, \norm{v}_{\spaceV}<\infty},
\end{align}
where $ \norm{v}_{\spaceV} $ is defined as:
\begin{align}\label{eq:def:V-norm}
	\norm{v}_{\spaceV} = \int_{\spatialD} \diffusivity(\xvec)|\grad v|^2 d\xvec.
\end{align}
Then the Galerkin formulation for the Poisson's equation presented in  \eqnref{eq:poisson-kl-intro} is to find $ u\in V $ such that,
\begin{align}\label{eq:galerkin-formulation-poisson-kl}
	B(u,w) = L(w) \ \ \forall w\in V,
\end{align}
where,
\begin{subequations}\label{eq:def:bilnear-linear-form}
	\begin{align}
		B(u,w) &= \int_{\spatialD}\diffusivity(\grad w\cdot \grad u) d\xvec \\
		L(w) &= \int_{\spatialD} wf d\xvec.
	\end{align}
\end{subequations}

\eqnref{eq:galerkin-formulation-poisson-kl} actually represents a number of equations each for a different test function $ w $ taken from the space $ \spaceV $. Thus, if $ \spaceV $ is discretized such that it can be represented by a finite basis, then the Galerkin formulation (\eqnref{eq:galerkin-formulation-poisson-kl}) yields a finite system of algebraic equations.

From the theory of variational calculus \cite{fox1987introduction, evans1998partial}, it is also known that \eqnref{eq:galerkin-formulation-poisson-kl} is the Euler-Lagrange equation of the following functional $ J(u) $ of $ u \in V $:
\begin{align} \label{eq:def:loss-functional-inner-product-form}
	J(u) = \half B(u,u) - L(u).
\end{align}
Therefore, the solution $ u $ can also be written as the minimizer of the cost function $ J $:
\begin{align}\label{eq:minimization-problem-continuous}
	u = \argmin_{u\in V} J(u).
\end{align}

In \neufenet{}, we make use of this functional $ J $, but instead of minimizing it against the solution $ u $, we minimize it against the network parameters $ \theta $. Since $ V $ is an infinite dimensional space, we need to discretize it to a finite subspace where we can evaluate $ J $. For this, let $ \mesh $ be a discretization of $ \spatialD $ into $ \nel $ finite elements $ K_i $ such that $ \cup_{\nel}K_i = D $. Then the discrete function space $ V^h $ is define as:
\begin{align}
	V^h = \setbuilder{v^h \in V: v^h|_{K}\in P_m(K), \ K\in \mesh},
\end{align}
where $ P_m(K) $ denotes the set of polynomial functions of degree $ m $ define on $ K $. Let the dimension of $ V^h $ be $ N $, which essentially means that the number of unknowns in the domain is also $ N $. Suppose $ \{ \phi_i(\xvec) \}_{i = 1}^{N} $ is a suitable basis that span $ V^h $. Then any function $ \uh \in V^h $ can be written as
\begin{align}\label{eq:def:fem-function-approximation-0}
	\uh(\xvec) = \sum_{i = 1}^{N} \phi_i(\mvec{x})U_i,
\end{align}
where $ U_i $ are the function values at the nodal points in the mesh $ \mesh $. 
Then \eqnref{eq:minimization-problem-continuous} can be rewritten for $ \uh $ as:
\begin{align}\label{eq:minimization-problem-discrete}
	\uh = \argmin_{\uh \in V^h} J(\uh).
\end{align}

This minimization of the energy functional in a finite dimensional space is commonly known as the Rayleigh-Ritz method. But in the presence of a neural network, we minimize $ J $ with respect to the network parameter $ \netparams $ instead of $ \uh $. For this, we need to first explicitly state the $ \theta $ dependence of $ \uh $. This can be done by a slight modification to \eqnref{eq:def:fem-function-approximation-0} as below:

\begin{align}\label{eq:def:fem-function-approximation-1}
	\uh(\xvec;\netparams) = \sum_{i = 1}^{N} \phi_i(\mvec{x})U_i(\theta),
\end{align}
along with a generalization of the function space $ \spaceVh $ as:
\begin{align}\label{eq:def:function-space-V-theta}
	V^{h'}_{\netparams} = \setbuilder{v(\theta)\in \spaceVh:v(x=0,y;\netparams) = 1, v(x=1,y;\netparams)=0}.
\end{align}
Then we can finally write down the \neufenet{} solution $ \uh_{\netparams^*} $ in two steps:
\begin{subequations}\label{eq:minimization-problem-theta-1}
	\begin{align}
		\netparams^* &= \argmin_{\netparams \in \Theta} J(\uh(\xvec; \netparams)) \\
		\uh_{\netparams^*} &=  \uh(\xvec; \netparams^*),
	\end{align}
\end{subequations}
where $ \uh_{\netparams^*} \in V^{h'}_{\netparams} $.

The discussion leading to \eqnref{eq:minimization-problem-theta-1} is based on a non-parametric diffusivity, i.e., $ \tilde{\diffusivity} = \diffusivity(\xvec) $. Extension of this formulation to the parametric PDE case is straightforward. Specifically, for a parameterized $ \omega $, i.e., $ \tilde{\diffusivity} = \diffusivity(\xvec,\omega) $, the function space $ V^{h'}_{\netparams}  $ is modified as:
\begin{align}\label{eq:def:function-space-V-theta-stochastic}
	V^h_\netparams = \setbuilder{v: \| v(\xvec,\omega; \netparams) \|_{V^h_{\netparams}} < \infty,\ v(x=0,y,\omega;\netparams) = 1, v(x=1,y,\omega;\netparams)=0}
\end{align}
where, $\|\cdot\|_{V^h_{\netparams}}$ is the energy norm
\begin{align}\label{eq:def:stochastic-norm}
	\|v\|_{V^h_{\netparams}}^2 = \mathbb{E}_{\omega\sim\Omega}\left[\int_D \nu(\mvec{x}, \omega; \netparams)|\nabla v|^2d\mvec{x}\right].
\end{align}
With this choice of function space, the \neufenet{} loss function can be written as:
\begin{align}\label{def:fem-loss-stochastic-expectation}
	L_{\neufenet{}}(\netparams) = \mathbb{E}_{\omega\in\Omega}J(\uh(\xvec, \omega; \netparams))
\end{align}

The right hand side of \eqnref{def:fem-loss-stochastic-expectation} involves two integrations: one over the spatial domain $ \spatialD $ and the other an expectation over $ \Omega $. The integration over $ \spatialD $ is evaluated numerically using Gaussian quadratures rules. And the expectation over $ \Omega $ is evaluated approximately by a summation over a finite number of samples, i.e.,
\begin{align}\label{def:fem-loss-stochastic-numerical}
	\hat{L}_{\neufenet{}}(\netparams) = \frac{1}{N_s}\sum_{i = 1}^{N_s}J(s(\omega_i), \uh(\xvec,\omega_i); \theta).
\end{align}

The loss function $ \hat{L}_{\neufenet{}}(\netparams) $ is now just a function of $ \netparams $ and we can minimize it with respect to $ \netparams $,
\begin{subequations}\label{eq:minimization-problem-theta-2}
	\begin{align}
		\label{eq:minimization-problem-theta-2-opt}
		\netparams^* &= \argmin_{\netparams \in \Theta} \hat{L}_{\neufenet{}}(\netparams) \\
		\label{eq:minimization-problem-theta-2-sol}
		\uh_{\netparams^*} &=  \uh(\xvec, \omega; \netparams^*).
	\end{align}
\end{subequations}

\begin{remark}\label{rem:uneural-notation}
	To simplify notations when we analyze errors in \secref{sec:errors}, we make a distinction between two representations of $ \netparams^* $: one of them is the theoretical minimum (denoted by $ \tilde{\netparams} $ in \secref{sec:errors}) and the other is the \emph{actual} set of parameters (denoted by $ \netparams $ in \secref{sec:errors}) obtained by optimizing \eqnref{eq:minimization-problem-theta-2-opt} with an optimization algorithm. This terminology then spawns two variants for $ \uh_{\netparams^*} $: $ \uNeuralExactOpt $ and $ \uNeural $ respectively.
\end{remark}

\subsection{Training algorithm for \neufenet{}}\label{sec:training-process}
We provide two versions of the training algorithm. (i) an algorithm for computing the solution for an instance of a PDE and (ii) an algorithm for approximating the solution for a parametric PDE. The model architecture and the loss function remain the same for both. For the instance version, we use a simple approach as explained in \algoref{alg:pde-instance}. While sampling from a distribution of coefficients/forcing field for a parametric PDE, we employ the mini-batch based optimization approach as explained in \algoref{alg:pde-parametric}. The sampling of the known quantities can be performed by using any random or qseudo-random sequence (see \secref{sec:results-poisson-kl-problem} for an example). For training the neural network, we predict the solution field using sampled inputs and compute the loss using the loss function derived above. We employ gradient descent based optimizers such as Adam~\citep{kingma2014adam} to perform the numerical optimization.

\begin{algorithm}[t]
	\caption{Algorithm for instance PDE solver}\label{alg:pde-instance}
	\begin{algorithmic}[1]
		\Require $ S^d $, $ \alpha $, \textsc{tol} and \texttt{max\_epoch} \Comment{$ \alpha $ = learning rate}
		\State Initialize $ \neuralMap $		
		\For{\texttt{epoch} $ \leftarrow $ 1 to \texttt{max\_epoch}}
		\State $U^d_\theta \gets G_{nn}(S^d)$ 
		\State Apply Dirichlet boundary conditions to $ U^d_\theta $
		\State  $loss = L_{\neufenet{}}(S^d, U^d_\theta)  $
		\State $ \theta \gets optimizer(\theta,\alpha,\grad_\theta(loss))$
		\If {$ loss < \textsc{tol}$}
		\State \texttt{break} 
		\EndIf
		\EndFor		 		
	\end{algorithmic}
\end{algorithm}
\begin{algorithm}[t]
	\caption{Algorithm for parametric PDE solver}\label{alg:pde-parametric}
	\begin{algorithmic}[1]
		\Require $ \{S_i^d\}_{i=1}^{N_s} $, $ \alpha $, \textsc{tol}  and \texttt{max\_epoch} \Comment{$ \alpha $ = learning rate}
		\State Initialize $ \neuralMap $		
		\For{\texttt{epoch} $ \leftarrow $ 1 to \texttt{max\_epoch}}
		\For{\texttt{mb} $ \leftarrow $ 1 to \texttt{max\_mini\_batches}}
		\State Sample $(S^d)_{mb}$ from the set $ \{S_i^d\}_{i=1}^{N_s} $
		\State $(U^d_\theta)_{mb} \gets G_{nn}\Big((S^d)_{mb}\Big)$
		\State Apply Dirichlet boundary conditions on $ (U^d_\theta)_{mb} $
		\State  $loss_{mb} = L_{\neufenet{}}((S^d_\theta)_{mb},(U^d_\theta)_{mb})  $
		\State $ \theta \gets optimizer(\theta,\alpha,\grad_\theta(loss_{mb})) $
		\EndFor
		\EndFor
	\end{algorithmic}
\end{algorithm}

\section{Implementations}\label{sec:implementations}

\begin{figure}
\centering
\begin{minipage}{0.14\textwidth}
\begin{tikzpicture}[scale=0.85]
\draw[step=2 cm] (0,0) grid (2,2);
\foreach \y in {0,2}{
\draw [fill=black](0,\y) circle (0.1cm);
\draw [fill=black](2,\y) circle (0.1cm);
}
\foreach \y in {0.423, 1.577}{
\draw [fill=red](0.423,\y) circle (0.1cm);
\draw [fill=red](1.577,\y) circle (0.1cm);
}
\draw (0.423, 1.25) node {1};
\draw (1.3,1.577) node {2};
\draw (0.7, 0.423) node {3};
\draw (1.577, 0.7) node {4};
\end{tikzpicture}
\end{minipage}
\begin{minipage}{0.35\textwidth}
\centering
\begin{tikzpicture}[scale=0.85]
\newcommand\reclenH{3}
\newcommand\reclenV{4}
\newcommand\ndiv{4}
\newcommand\stepsize{1}
\newcommand\xl{0}
\newcommand\yl{0}
\newcommand\osx{-0.5}
\newcommand\osy{0.5}
\draw[step=\stepsize cm] (\xl,\yl) grid (\xl+\reclenH,\yl+\reclenV);

\foreach \y in {0,1,2,3,4}{
\draw [fill=black](0,\y) circle (0.1cm);
\draw [fill=black](1,\y) circle (0.1cm);
\draw [fill=black](2,\y) circle (0.1cm);
\draw [fill=black](3,\y) circle (0.1cm);
}

\draw (\osx, 0) node {$e$};
\draw (\osx, 1) node {$d$};
\draw (\osx, 2) node {$c$};
\draw (\osx, 3) node {$b$};
\draw (\osx, 4) node {$a$};

\draw (0,\reclenV + \osy) node {$1$};
\draw (1,\reclenV + \osy) node {$2$};
\draw (2,\reclenV + \osy) node {$3$};
\draw (3,\reclenV + \osy) node {$4$};

\foreach \y in {0.2115, 0.7885, 1.2115, 1.7885, 2.2115, 2.7885, 3.2115, 3.7885}{
\draw [fill=red](0.2115,\y) circle (0.08cm);
\draw [fill=red](0.7885,\y) circle (0.08cm);
\draw [fill=red](1.2115,\y) circle (0.08cm);
\draw [fill=red](1.7885,\y) circle (0.08cm);
\draw [fill=red](2.2115,\y) circle (0.08cm);
\draw [fill=red](2.7885,\y) circle (0.08cm);
}

\foreach \y in {0.7885, 1.7885, 2.7885, 3.7885}{
\draw [fill=green](0.2115,\y) circle (0.08cm);
\draw [fill=green](1.2115,\y) circle (0.08cm);
\draw [fill=green](2.2115,\y) circle (0.08cm);
}

\end{tikzpicture}		
\end{minipage}	
\caption {(Left) A single 2D element in FEM, with black dots denoting ``nodes" and red dots denoting $ 2\times 2 $ Gauss quadrature points. (Right) A finite element mesh, with $ 4\times 3 $ linear elements and $ 5\times 4 $ nodes. Each of these elements contains Gauss points for integration to be performed within that element. Within each element, the ``first" quadrature point (marked ``1" on left) is marked green, and others red.}
\label{fig:fem-mesh-with-gauss-pt}
\end{figure}
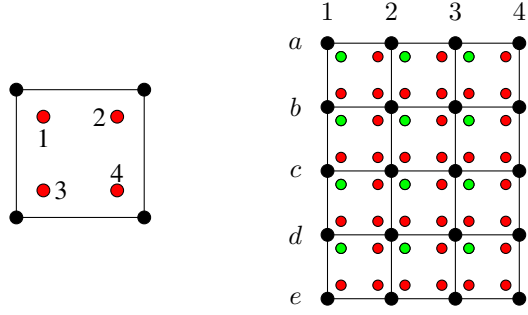
																							
\begin{figure}[!htb]
\centering
\begin{tikzpicture}
\tikzset{%
parenthesized/.style={%
left delimiter  = (,
right delimiter = ),
},
node distance = 5mu,
}


\matrix[matrix of math nodes, parenthesized] (I) {
a_1 & a_2 & a_3 & a_4 \\
b_1 & b_2 & b_3 & b_4 \\
c_1 & c_2 & c_3 & c_4 \\
d_1 & d_2 & d_3 & d_4 \\
e_1 & e_2 & e_3 & e_4 \\
};
\node (*) [right = of I] {${}*{}$};

\matrix[matrix of math nodes, parenthesized] (K) [right = of *] {
N_1 & N_2 \\
N_3 & N_4 \\
};

\node (=) [right = of K] {${}={}$};

\matrix[matrix of math nodes, parenthesized] (I*K) [right = of {=}] {
\color{green}\bullet & \color{green}\bullet & \color{green}\bullet\\
\color{green}\bullet & \color{green}\bullet & \color{green}\bullet\\
\color{green}\bullet & \color{green}\bullet & \color{green}\bullet\\
\color{green}\bullet & \color{green}\bullet & \color{green}\bullet\\
};
\newcommand\rowResult{2}
\newcommand\colResult{1}

\begin{scope}[on background layer]
\newcommand{\padding}{2pt}
\coordinate (Is-nw) at ([xshift=-\padding, yshift=+\padding] I-\rowResult-\colResult.north west);
\coordinate (Is-se) at ([xshift=+\padding, yshift=-\padding] I-\the\numexpr\rowResult+\numRowsK-2\relax-\the\numexpr\colResult+\numColsK-2\relax.south east);
\coordinate (Is-sw) at (Is-nw |- Is-se);
\coordinate (Is-ne) at (Is-se |- Is-nw);

\filldraw[red,   fill opacity=.1] (Is-nw) rectangle (Is-se);
\filldraw[green, fill opacity=.1] (I*K-\rowResult-\colResult.north west) rectangle (I*K-\rowResult-\colResult.south east);

\draw[blue, dotted] 
(Is-nw) -- (K.north west)
(Is-se) -- (K.south east)
(Is-sw) -- (K.south west)
(Is-ne) -- (K.north east)
;
\draw[green, dotted] 
(I*K-\rowResult-\colResult.north west) -- (K.north west)
(I*K-\rowResult-\colResult.south east) -- (K.south east)
(I*K-\rowResult-\colResult.south west) -- (K.south west)
(I*K-\rowResult-\colResult.north east) -- (K.north east)
;

\draw[blue,  fill=blue!10!white] (K.north west) rectangle (K.south east);
\end{scope}

\tikzset{node distance=0em}
\node[below=of I] (I-label) {$ (U^d_\theta)_M $};
\node at (K |- I-label)     {$ K_{GP1} $};
\node at (I*K |- I-label)   {$ ((U^d_\theta)_{GP_1})_M $};
\end{tikzpicture}%
\caption{Quadrature quantity evaluation in FEM context. $ (U^d_\theta)_M $ is the matrix view of the nodal values.$ K_{GP1} $ is kernel containing the basis function values at ``gauss point - 1" (top left corner). This convolution results in the function values evaluated at the Gauss point ``1" of each element (marked green). $ (U^d_\theta{}_{GP1})_M $ is the matrix of this result. Function values (or their derivatives) evaluated at Gauss  points can then be used in any integral evaluation. For example, $ \int u^h dD = |J| \sum_{I\in M} \left[\sum_{i=1}^{4}(w_i (U^d_\theta)_{GPi})_M\right] $, where $ |J| $ is the transformation Jacobian for integration and $ w $ are the quadrature weights.}
\label{fig:fem-function-evaluation-at-single-gauss-pt}
\end{figure}

\begin{figure}[!htb]
	\centering
	\includegraphics[width=0.6\linewidth]{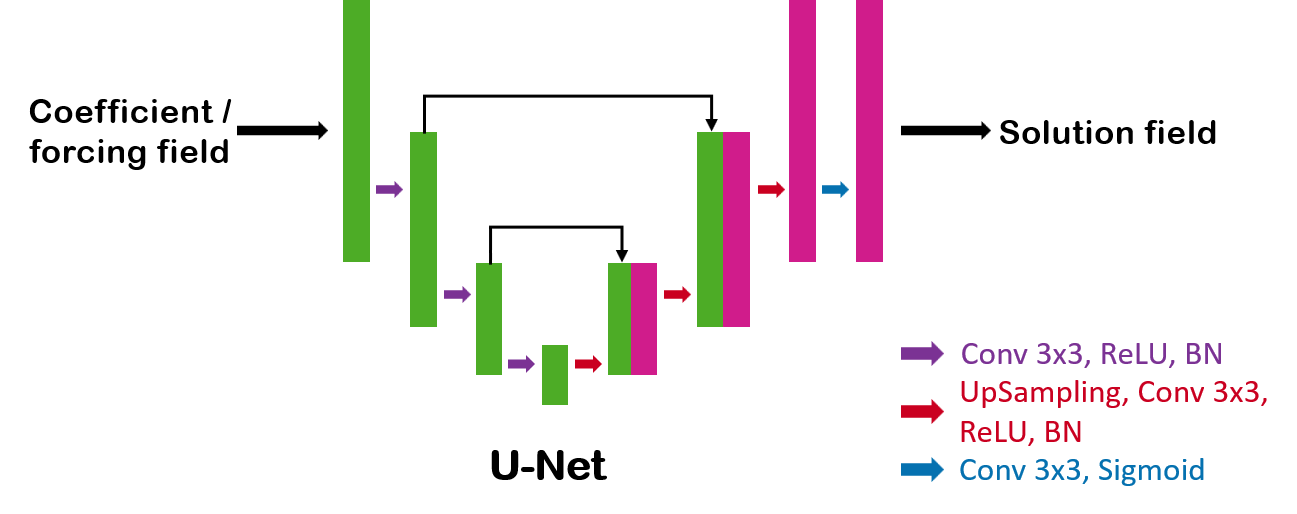}
	\caption{UNet architecture used for training \neufenet{}.}
	\label{fig:unetarch}
\end{figure}

\subsection{Model architecture for \neufenet{}}\label{sec:model-arch}
Due to the structured grid representation of $\mathcal{S}^d$ and similar structured representation of $U^d_\theta$, deep convolutional neural networks are a natural choice of network architecture. The spatial localization of convolutional neural networks helps in learning the interaction between the discrete points locally. Since the network takes an input of a discrete grid representation (similar to an image, possibly with multiple channels) and predicts an output of the solution field of a discrete grid representation (similar to an image, possibly with multiple channels), this is considered to be similar to an image segmentation or image-to-image translation task in computer vision. U-Nets~\citep{ronneberger2015u,cciccek20163d} have been known to be effective for applications such as semantic segmentation and image reconstruction. Due to its success in diverse applications, we choose U-Net architecture for \neufenet{}. The architecture of the network is shown in \figref{fig:unetarch}. First, a block of convolution, and instance normalization is applied. Then, the output is saved for later use via skip-connection. This intermediate output is then down sampled to a lower resolution for a subsequent block of convolution, instance normalization layers. This process is continued for two more times. Now, the upsampling starts where, the saved outputs of similar dimensions are concatenated with the output of upsampling for creating the skip-connections followed by a  convolution layer. LeakyReLU activation was used for all the intermediate layers. The final layer has a Sigmoid activation.

\subsection{Applying boundary conditions}\label{sec:boundary-conditions}
In \neufenet{}, the Dirichlet boundary conditions are applied exactly. The output $ U^d_\theta $ does not contain the boundary conditions. Thus, a small post-processing step is done to the network output to force the Dirichlet boundary conditions onto the respective boundaries. This can be done in a differentiable manner in modern machine learning software libraries such as PyTorch \citep{paszke2019pytorch}. This exact imposition of Dirichlet boundary conditions allow the training process to be much smoother and interpretable, because there is no penalty term involved in the loss function. Thus the loss function retains its convex nature with respect to the solution $ \uh $.

On the other hand, Neumann conditions are included in the variational form of the PDE right at the continuous level. Especially zero-Neumann conditions are exactly satisfied at the discrete level without requiring us to do anything (``do-nothing'' conditions). 

\subsection{Calculation of derivatives and integration}\label{sec:derivative-integal-calc-detail}


The full domain integration (i.e., $ \int D $) is nothing but the simple sum of the integration over the individual elements (i.e., $ \sum_{i = 1}^{N_{el}} \int D_i $). This integration over an individual element is in turn the simple weighted sum of the integrand evaluated at the Gauss quadrature points. This evaluation at a single Gauss point can be represented as convolution. Thus, if there are 4 Gauss points in each element, then 4 convolution operations will evaluate the integrand at those points for each element. After that, we only need to sum across Gauss points first, then followed by a sum across elements. See \figref{fig:fem-mesh-with-gauss-pt} and \figref{fig:fem-function-evaluation-at-single-gauss-pt} for visualized representation of this process.

\begin{remark}
	Since the result of the integration process is a scalar loss value, there is no requirement to compute, store or assemble a matrix. 
\end{remark}

\section{Error analysis} \label{sec:errors}
\subsection{Error analysis for the instance case}
We provide estimates on the errors incurred by \neufenet{} in approximating the solution. Suppose the exact solution of \eqnref{eq:poisson-kl-intro} is $ \uExact $ and the solution obtained at the end of the training process is given by $ \uNeural $ (see \remref{rem:uneural-notation}). Let us also define $ \uNeuralExactOpt $ as the best possible function in $ \spaceVtheta  $. Note that this function may or may not be able to match $ \uDiscrete $, but it represents the best possible function that the neural network function class can produce. Note also that $ \uNeuralExactOpt $ might be different than $ \uNeural $, especially if the optimization algorithm is not able to reach the optimum $ \tilde{\netparams} $. We try to bound the error $ \uNeural - \uExact $ by first breaking down the total error into errors from different sources. Theorem \ref{thm:diffnet-errors-non-parametric} is our main result for the single instance PDE version, while Thoerem \ref{thm:diffnet-errors-parametric} is our main reuslt for the parametric PDE version. 
\begin{lemma} \label{thm:opt-error-low-single-instance}
	Let $ \uNeural $ be the solution of \eqnref{eq:def:loss-functional-inner-product-form} when it is optimized by \algoref{alg:pde-instance}. Then the optimization error  $ \normV{\eOpt} = \sqrt{2\norm{\errJ}_{1}}$, where $ \eOpt = \uNeural - \uNeuralExactOpt $ and $ \errJ = J(\uNeural) - J(\uNeuralExactOpt) $.
\end{lemma}
\begin{proof}
	\begin{align}
	\begin{split}
	J(\uNeural) &= \int_{\spatialD} \left[\half \diffusivity |\grad \uNeural |^2 - f\uNeural \right]d\xvec \\
	&= \int_{\spatialD} \left[\half \diffusivity |\grad \uNeuralExactOpt + \grad\eOpt |^2 - f(\uNeuralExactOpt + \eOpt) \right]d\xvec \ \ (\text{using}\ \eOpt=\uNeural-\uNeuralExactOpt) \\
	&= \int_{\spatialD} \left[\half \diffusivity |\grad \uNeuralExactOpt|^2 + \half \diffusivity |\grad \eOpt |^2 +\diffusivity\grad\uNeuralExactOpt\cdot\grad\eOpt - f(\uNeuralExactOpt + \eOpt) \right]d\xvec \\
	&= \int_{\spatialD}\left[ \half \diffusivity|\grad \uNeuralExactOpt |^2 - f\uNeuralExactOpt \right]d\xvec + \half\int_{\spatialD}\diffusivity|\grad\eOpt|^2 d\xvec + \int_{\spatialD}\left[ \diffusivity\grad\eOpt\cdot\grad\uNeuralExactOpt - f\eOpt \right]d\xvec \\
	&= J(\uNeuralExactOpt) + \half\normV{\eOpt}^2
	\end{split}
	\end{align}
	where in the final step, we have used the definitions of $ J $ and $ \normV{\eOpt} $ along with the fact that $ \int_{\spatialD}\left[ \diffusivity\grad\eOpt\cdot\grad\uNeuralExactOpt - f\eOpt \right]d\xvec = 0 $ since $ \eOpt \in \spaceVh$ (see \eqnref{eq:galerkin-formulation-poisson-kl} and \eqnref{eq:def:bilnear-linear-form}). Therefore, we have
	\begin{align}
		\normV{\eOpt}^2 = 2\left[J(\uNeural) - J(\uNeuralExactOpt)\right] = 2\errJ = 2\norm{\errJ}_{1},
	\end{align}
	since $ \errJ = J(\uNeural) - J(\uNeuralExactOpt) >0 $.
\end{proof}

We next describe a theorem that provides an estimate on the network capacity.
\begin{theorem}\label{thm:network-approximation-capacity}
	Fix $p \in \Omega$ and consider a PDE as defined in \eqnref{pde:abstract-representation-continuous} over a compact domain $D$ which is uniformly discretized with resolution $h$. Let $u^*_h$ be the true solution evaluated at the grid points. Consider the hypothesis class:
	\begin{align}
		\gH := \{ u_\Theta : p \mapsto \sum_{l=1}^k a^{i}_l \text{ReLU}(\langle w^{i}_l, p \rangle + b^{i}_l), i=1,\ldots,n. \}
	\end{align}
	defined as the set of all two-layer neural networks with $k$ hidden neurons equipped with ReLU activation. Then, as long as $k = \Omega(1/h)$, there exists a network in $\gH$ for which $\text{Err}_\gH = 0$.
\end{theorem}
\begin{proof}
	The proof follows a recent result by \citet{bubeck2020-memorization}. Let the output resolution of the network be $n$. For a fixed $p$, consider the (linear) vector space spanned by all possible (perhaps uncountably many) basis functions of the form
	$$
	f(\cdot) = a^i_l \text{ReLU}(\langle w^i_l, \cdot \rangle + b_l)
	$$
	where $a,b,w$ are arbitrary real-valued weights. Since $i = 1, \ldots, n$, the span of this space is no more than $n$-dimensional and isomorphic to $\mathbb{R}^m$ for some $m \leq n$. Therefore, there is a set of no more than $n$ basis functions (i.e., $n$ neurons) that can be used to represent $u^*_{p}$ any fixed $p$. Assuming the dimension of $p$ is small, we have $n \lesssim 1/h$. Therefore, $k = \Omega(1/h)$ neurons are sufficient to reproduce $u^*_p$.
\end{proof}

Notice that the above theorem shows that there exist \neufenet{} architectures that exactly drive the modeling error down to zero. However, the proof is \emph{non-constructive}, and there is no obvious algorithm to find the basis functions that reproduce the solution at the evaluation points. \thmref{thm:network-approximation-capacity} essentially allows us to choose the neural network parameter family $ \Theta $ such that the modeling error $ \eCap $ is low. Since, we are free to choose the network architecture, we can always assume ( and \textit{a posteriori} confirm) that
\begin{align}
	\norm{\eCap} = \norm{\uNeuralExactOpt - \uDiscrete} \leq \epsilon, \ \ \epsilon>0.
\end{align}
\begin{remark}
	\neufenet{} is designed to be agnostic to a neural network. Therefore, both fully connected neural network as well as convolutional neural networks can be used for the network approximation. Since a convolutional neural network can be interpreted as a special case of a fully connected network with sparse weights \cite{lecun1998gradient}, the above estimate still holds.
\end{remark}

For the third source of error, i.e., the error due to discretization using finite element method can be estimated from standard finite element analysis literature. We start with the following assumption:
\begin{assumption}\label{assump:discretization}
	Assume that the spatial domain $ \spatialD $ is discretized by a mesh $ \mesh $ that consists of hyperrectangular elements. Each element $ K \in \mesh$ has a bounded radius, i.e., $ 0 < h_{min} \leq r(K) \leq h_{max} < \infty $. We define the mesh length $ h = \min \{r(K)\}_{i = 1}^{n_{el}} $
\end{assumption}
\begin{lemma}
	The exact solution to \eqnref{eq:galerkin-formulation-poisson-kl}, $ u \in  H^2(\spatialD) $.
\end{lemma}
\begin{proof}
	In \eqnref{eq:galerkin-formulation-poisson-kl}, the diffusivity $ \diffusivity(\xvec; \omega) \in C(\spatialD)$ for any fixed $ \omega \in \Omega $. Furthermore, the forcing function $ f = 0 \in L^2(\spatialD) $. Using results from regularity theory (such as \citep{evans1998partial}, Sec. 6.3 Theorem 1), we conclude $ u\in H^2(\spatialD) $.
\end{proof}
\begin{lemma}\label{thm:discretization-error}
	Let \cref{Assumption}{assump:discretization} hold. Further assume that the basis functions $ \phi_i(\xvec) $ in \eqnref{eq:def:fem-function-approximation-0} are chosen such that $ \phi_i(\xvec)\in C^{\alpha}(K) $ ($ \alpha\geq 1 $) locally within each element $ K\in\mesh $. Then,
	\begin{align}
		\norm{\uDiscrete - \uExact}_{L^2(\spatialD)} \leq C_d h^{\alpha+1},
	\end{align}
	where $ C_d = C_d(D, |u|_{H^2}) $ is a constant.
\end{lemma}
\begin{proof}
	We refer to standard texts such as \citet{oden2012introduction}(Sec. 8.6, Theorem 8.5) or  \citet{brenner2007mathematical}(Sec. 5.7) for the proof.
\end{proof}

Finally we can write down the estimate for the generalization error in \neufenet{} in the form of the following theorem.
\begin{theorem} \label{thm:diffnet-errors-non-parametric}
	Consider any \neufenet{} architecture producing fields with grid spacing $h$. Let $\gH$ denote the hypothesis class of all networks obeying that architecture. Suppose that $\Theta \in \gH$ is a \neufenet{} trained using the loss function $ J $ defined in \eqnref{eq:minimization-problem-continuous}. Then, its generalization error obeys:
	\begin{align} \label{eq:error-bound-non-parametric}
	\normVtheta{\eGen} \leq \text{Err}_\Theta + \text{Err}_\gH + O(h^{\alpha+1})
	\end{align}
	where $\text{Err}_\Theta$ is a term that only depends on the \neufenet{} optimization procedure and $\text{Err}_\gH$ only depends on the choice of hypothesis class $\gH$. The $ \alpha $ in the third term is the local degree of the basis functions as in \lemref{thm:discretization-error}.
\end{theorem}
\begin{proof}
	The result is simply an application of the triangle inequality. Dropping the subscript $ \spaceVtheta $:
	\begin{align}
	\begin{split}
	\norm{\eGen} &= \norm{\uNeural - \uNeuralExactOpt + \uNeuralExactOpt - \uDiscrete + \uDiscrete - \uExact} \\
	&\leq \norm{\uNeural - \uNeuralExactOpt} + \norm{\uNeuralExactOpt - \uDiscrete} + \norm{\uDiscrete - \uExact} \\
	&= \norm{\eOpt} + \norm{\eCap} + \norm{\eDis} \\
	&= \text{Err}_\Theta + \text{Err}_\gH + O(h^{\alpha+1})
	\end{split}
	\end{align}	
	Using \lemref{thm:opt-error-low-single-instance}, the optimization error Err$ _{\Theta} $ is nothing but $ \sqrt{2 \norm{\errJ}_{1}} $.
\end{proof}

\subsection{Extending the error analysis for the parametric version}
The above theorem is for a single parameter choice $p\in\Omega$. An identical argument can be extended to the loss constructed by sampling a \emph{finite} number (say $m$) of parameters from a distribution over $\Omega$. We obtain the following corollary from \thmref{thm:network-approximation-capacity}:
\begin{corollary}
	Consider a finite-sample version of the loss $\hat{L}_{\neufenet{}}$ constructed by taking the average over $m$ parameter choices sampled from $\Omega$. Consider the hypothesis class:
	\begin{align}
		\gH := \{ u_\Theta : p \mapsto \sum_{l=1}^k a^{i}_l \text{ReLU}(\langle w^{i}_l, p \rangle + b^{i}_l), i=1,\ldots,n. \}
	\end{align}
	defined as the set of all two-layer neural networks with $k$ hidden neurons equipped with ReLU activation. Then, as long as $k = \Omega(m/h)$, there exists a network in $\gH$ for which $\text{Err}_\gH = 0$.
\end{corollary}


This corollary shows that a wide two-layer network exists that can reproduce any (finite) set of field solutions to a PDE system, as long as the width scales linearly in the cardinality of the set. Getting bounds independent of the cardinality is an interesting open question. We have the following result for the optimization error:

\begin{lemma}
	Let $ \uNeural $ be the solution of the optimization problem in \eqnref{eq:minimization-problem-theta-2} when it is optimized by \algoref{alg:pde-parametric}. Then the optimization error $ \norm{\eOpt}_{\spaceVhtheta} = \sqrt{2\norm{\errL}_{1}} $, where $ \eOpt = \uNeural - \uNeuralExactOpt $, $ \errL = \hat{L}_{\neufenet{}}(\netparams) - \hat{L}_{\neufenet{}}(\netparamsExact)$.
\end{lemma}

\begin{proof}
	Starting with the definition of $ \hat{L} $ from \eqnref{def:fem-loss-stochastic-numerical} and denoting $ \diffusivity_i = \diffusivity(\xvec,\omega_i) $ and $ u_{\netparams, i} = u(\xvec, \omega_i; \netparams) $: 
	\begin{align}
		\hat{L}_{\neufenet{}}(\netparams) = \oneOver{N_s}\sum_{i = 1}^{N_s} J(\diffusivity_i, u_{\netparams, i}) = \oneOver{N_s}\sum_{i = 1}^{N_s} \left( J(\diffusivity_i, u_{\netparamsExact, i}) + \half \norm{e_{\netparams, i}}_{\spaceVh}^2 \right) = \hat{L}_{\neufenet{}}(\netparamsExact) + \half \norm{\eOpt}^2_{\spaceVhtheta},
	\end{align}
	where we have used $ \norm{\eOpt}^2_{\spaceVhtheta} = \mathbb{E}\normV{\eOpt} = \oneOver{N_s}\sum_{i = 1}^{N_s}\normV{e_{\netparams, i}}^2 $. Denoting $ \errL = \hat{L}_{\neufenet{}}(\netparams) - \hat{L}_{\neufenet{}}(\netparamsExact) >0 $, we have the result.
\end{proof}


With this, an analogue of \thmref{thm:diffnet-errors-non-parametric} can be stated for the parametric training with finite data as below.
\begin{theorem} \label{thm:diffnet-errors-parametric}
	Consider any \neufenet{} architecture producing fields with grid spacing $h$. Let $\gH$ denote the hypothesis class of all networks obeying that architecture. Suppose that $\Theta \in \gH$ is a \neufenet{} trained using the loss function $ \hat{L}_{\neufenet{}} $ defined in \eqnref{def:fem-loss-stochastic-numerical}. Then, its generalization error obeys:
	\begin{align} \label{eq:error-bound-parametric}
	\normVhtheta{\eGen} \leq \text{Err}_\Theta + \text{Err}_\gH + O(h^{\alpha+1})
	\end{align}
	where $\text{Err}_\Theta$ is a term that only depends on the \neufenet{} optimization procedure; $\text{Err}_\gH$ only depends on the choice of hypothesis class $\gH$ and $ \alpha $ is the local degree of the basis function as in \lemref{thm:discretization-error}.
\end{theorem}

\begin{proof}
	The result is simply an application of the triangle inequality. Dropping the subscript $ \spaceVtheta $:
	\begin{align}
	\begin{split}
	\normVhtheta{\eGen} &= \normVhtheta{\uNeural - \uNeuralExactOpt + \uNeuralExactOpt - \uDiscrete + \uDiscrete - \uExact} \\
	&\leq \normVhtheta{\uNeural - \uNeuralExactOpt} + \normVhtheta{\uNeuralExactOpt - \uDiscrete} + \normVhtheta{\uDiscrete - \uExact}\\
	&\leq \normVhtheta{\eOpt} + \normVhtheta{\eCap} + \normVhtheta{\eDis}.
	\end{split}
	\end{align}	
	The third norm can be estimated as
	\begin{align*}
		\normVhtheta{\eDis} = \normVhtheta{\uDiscrete - \uExact} 
		&= \sqrt{\mathbb{E}_{\omega \sim \Omega} \norm{\uDiscrete_i(\omega) - \uExact(\omega)}_{V}^2 } 
		=  \sqrt{\oneOver{N_s}\sum_{i=1}^{N_s}\norm{\uDiscrete_i-\uExact_i}_{V}^2 } 
		\leq \sqrt{\oneOver{N_s}\sum_{i=1}^{N_s}\norm{\diffusivity_i}_{L^2(D)}^2 \norm{\grad(\uDiscrete_i-\uExact_i)}_{L^2(D)}^2 }  \\
		&\leq \sqrt{\oneOver{N_s}\sum_{i=1}^{N_s}\norm{\diffusivity_i}_{L^2(D)}^2 \Bigg(\sum_{K\in\mesh}\norm{\grad(\uDiscrete_i-\uExact_i)}_{L^2(K)}^2\Bigg) } 
		\leq \sqrt{\oneOver{N_s}\sum_{i=1}^{N_s}\norm{\diffusivity_i}_{L^2(D)}^2 C^h_{I}\Bigg(\sum_{K\in\mesh}\norm{(\uDiscrete_i-\uExact_i)}_{L^2(K)}^2\Bigg) }  \\
		&= \sqrt{\oneOver{N_s}\sum_{i=1}^{N_s}C^h_{I}\norm{\diffusivity_i}_{L^2(D)}^2 \norm{(\uDiscrete_i-\uExact_i)}_{L^2(D)}^2 }
		\leq \sqrt{\oneOver{N_s}\sum_{i = 1}^{N_s} C_{di}h^{2(\alpha+1)}}
		\leq C_e h^{\alpha+1},
	\end{align*}	
	where $ C_e = C_e \left(\hat{\Omega}_{N_s}, D, C^h_I, A, B \right) $. Here $ C^h_I $ is the discrete inverse Poincar\'e constant such that $ \norm{\grad u}_{L^2(K)} \leq C^h_I \norm{u}_{L^2(K)}$ for all $ K\in\mesh $; $ A = \max\{\norm{\diffusivity_i}_{L^2(\spatialD)}\}_{i=1}^{N_s} $ and $ B = \max \{|u_i|_{H^2(\spatialD)}\}_{i=1}^{N_s} $.
\end{proof}

\section{Results}\label{sec:results}

\subsection{Error convergence for a Poisson problem} \label{sec:linear-poisson-manufactured}

At the outset, we would like to validate the error bounds stated in \secref{sec:errors}. To this end, we solve the following non-parametric Poisson's equation:
\begin{subequations}\label{eq:poisson-manufactured-pde}
	\begin{align}
	-\laplacian u &= f \ \ \text{in} \ \spatialD \subset \mathbb{R}^2 \\
	u &= 0 \ \ \text{on} \ \partial\spatialD,
	\end{align}
\end{subequations}
where $ D = [0,1]^2 $ is a two-dimensional square domain. The forcing is chosen as
\begin{align} \label{eq:poisson-mms-forcing}
	f = f(\mvec{x}) = f(x,y) = 2\pi^2\sin(\pi x) \sin(\pi y).
\end{align}
The exact solution to \eqnref{eq:poisson-manufactured-pde} with the forcing function shown in \eqnref{eq:poisson-mms-forcing} is given by $ u_{ex}(x,y) = \sin(\pi x) \sin(\pi y) $. For solving this problem, we seek to train a \neufenet{} that can predict the solution $ u $ given the forcing $ f $. At the discrete level, the network takes input $ F^d $ and outputs $ U^d $. The loss function \eqnref{eq:def:loss-functional-inner-product-form} takes the concrete form:
\begin{align}\label{def:loss-fem-poisson-manufactured}
J(u) = \frac{1}{2} \int |\grad u|^2 d\mvec{x} - \int u f d\mvec{x}.
\end{align}
\figref{fig:poisson-manufactured-contours} shows the contours of the network input $ F^d $, output $ U^d $, the exact solution evaluated on the mesh $ \mesh $ (a $ 256\times 256 $ grid) and the error ($ U^d - U_{ex} $).

\begin{figure}[!htb]
	\centering
    \includegraphics[trim=120 0 50 0,clip,width=0.99\linewidth]{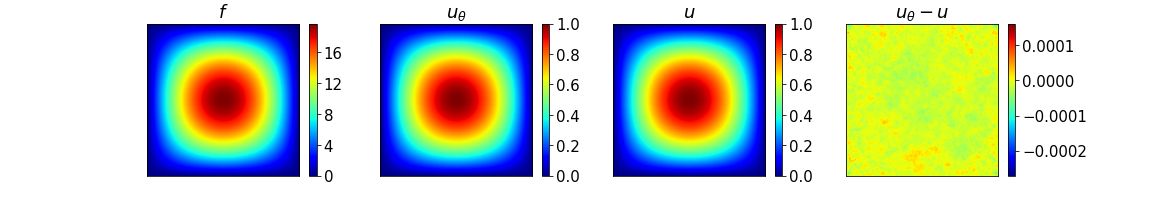}
	\caption{Solution to \eqnref{eq:poisson-manufactured-pde} with the forcing term as shown in  \eqnref{eq:poisson-mms-forcing} on a $ 256\times 256  $ grid. (Left) the discrete forcing $ F^d $, (middle-left) the solution $ U^d $ obtained from a \neufenet{} with a U-net architecture, (middle-right) the exact solution $ u_{ex} = \sin(\pi x)\sin(\pi y) $ evaluated on the mesh, denoted $ U_{ex}^d $, (right) the error $ (U^d - U^d_{ex}) $.  From left to right: $f$, $u^h_\theta$, $u_{num}$ and ($u^h_\theta-u_{num}$). Note: $ \norm{\uNeural} = 0.499979 $, $ \norm{u} = 0.5 $, $ \norm{\uNeural-u} = 2.4\times 10^{-5} $.}
	\label{fig:poisson-manufactured-contours}
\end{figure}

To study convergence behavior, we repeat this procedure with varying mesh length $ h $, starting from $ h = \oneOver{8} $ and gradually decreasing $ h $ till $ h = \oneOver{256} $. The calculated errors for each resolution are reported in \figref{fig:convergence-plot-non-parametric}. 

\thmref{thm:diffnet-errors-non-parametric} estimates the total error by the contributions from three individual sources of errors, namely, the error due to finite element discretization ($ \eDis $), the error due to network approximation capacity ($ \eCap $) and the error due to the optimization process ($ \eOpt $). Of these three errors, $ \eDis $ maintains a particular relation to the discretization parameter $ h $ (i.e., the mesh size). The other two errors do not hold a straight-forward relation with the mesh size or the network parameters. Thus, we put the estimate in \eqnref{eq:error-bound-non-parametric} to test by keeping both $ \eCap $ and $ \eOpt $ sufficiently low, so that the dominating error is the discretization error $ \eDis $.

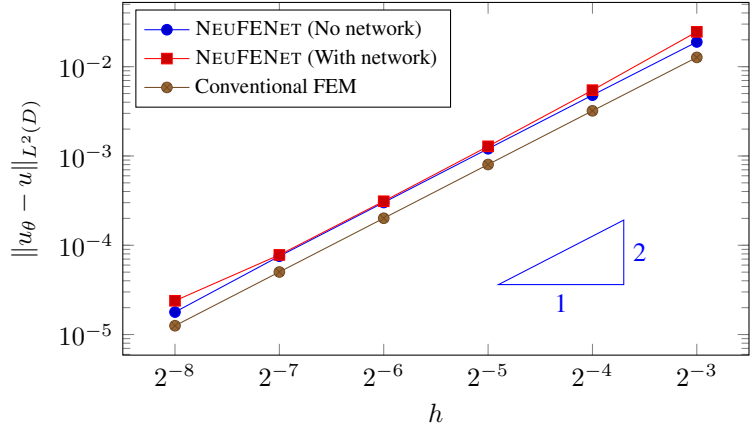
\begin{figure}[!t]
\centering
  \begin{tikzpicture}
     \begin{loglogaxis}[
         width=0.6\linewidth, 
         height=0.38\linewidth,
         xlabel=$h$, 
         ylabel=$\norm{\uNeural - u}_{L^2(\spatialD)}$,
         legend style={at={(0.02,0.98)},anchor=north west,legend columns=1}, 
         x tick label style={rotate=0,anchor=north}, 
         xtick={0.00390625, 0.0078125, 0.015625 , 0.03125  , 0.0625, 0.125},
         xticklabels={$2^{-8}$,$2^{-7}$,$2^{-6}$,$2^{-5}$,$2^{-4}$,$2^{-3}$},
         legend cell align={left},
      ]
      \addplot 
      table[x expr={1.0/((\thisrow{elm}))},y expr={\thisrow{diffnet_no_network}},col sep=comma]{results_data/poisson-convergence/errors_linear.txt};
      \addplot 
      table[x expr={1.0/((\thisrow{elm}))},y expr={\thisrow{diffnet_with_network_capacity_enhanced}},col sep=comma]{results_data/poisson-convergence/errors_linear.txt};
      \addplot 
      table[x expr={1.0/((\thisrow{elm}))},y expr={\thisrow{conventional_fem_code}},col sep=comma]{results_data/poisson-convergence/errors_linear.txt};
      \logLogSlopeTriangle{0.8}{0.2}{0.2}{2}{blue};
      \legend{\footnotesize \neufenet{} (No network), \footnotesize \neufenet{} (With network), \footnotesize Conventional FEM}
     \end{loglogaxis}
  \end{tikzpicture}
\caption{Convergence of the error in $L^2$ norm for the Poisson's equation with analytical solution $u =\sin{(\pi x)}\sin{(\pi y}$).}
\label{fig:convergence-plot-non-parametric}
\end{figure}

We illustrate this by running two sets of \neufenet{} simulations.
\begin{itemize}
	\item The first set is where we completely remove the neural network and optimize the loss directly with respect to the function $ \uh $, i.e.,
	\begin{align}
		\uh = \argmin_{\uh \in \spaceVh} J(\uh),
	\end{align}
	which is nothing but a classic Rayleigh-Ritz optimization. The absence of a neural network eliminates the error $ \eCap $. So the total error is a combination of only $ \eDis $ and $ \eOpt $. The total errors from this algorithm are plotted against the mesh size  $ h $ in \figref{fig:convergence-plot-non-parametric} under the legend ``\neufenet{} (No network)''. We see that the error plot has a slope of 2.
	
	\item The second set uses a neural network. Thus the optimization statement is the same as presented in \eqnref{eq:minimization-problem-theta-1}. Therefore, as discussed in \thmref{thm:diffnet-errors-non-parametric}, the total error is a combination of all three errors indicated in \eqnref{eq:error-bound-non-parametric}. To keep $ \eCap $ negligible, we need to be cognizant of the fact that the spatial degrees of freedom (i.e., the number of  bases in $ \spaceVh $) is inversely proportional to $ h^2 $ (for a 2D domain). Thus, as we decrease $ h $, the ``size" of $ \spaceVh $ increases. Therefore the function space $ \spaceVhtheta $ must also get bigger, in particular, it should be big enough to satisfy $ \spaceVhtheta\supset\spaceVh $. One way to accommodate this fact is to use a high-capacity network to solve the equation at all $ h $-levels. But the UNet architecture described in \secref{sec:model-arch} does not allow very high depth when the input/output size is low. Thus, we must enhance the network capacity at different $ h $ levels gradually by increasing the network depth. Using this strategy, we solve \eqnref{eq:poisson-manufactured-pde} at various $ h $-levels and plot the errors in \figref{fig:convergence-plot-non-parametric}. The slope of 2 of the error curve once again confirms that both $ \eCap\approx0 $ and $ \eOpt\approx0 $. Interestingly, if we do not keep $ \eCap $ negligible by increasing the depth of the network for smaller $h$, both $ \eCap $ and $ \eOpt $ can dominate (see \ref{app:error-convergence}). This suggests that gradual increase in network complexity is warranted as the discretization becomes finer. This is made computationally efficient by using multi-grid like approaches~\cite{balu2021distributed}.
\end{itemize}
\figref{fig:convergence-plot-non-parametric} also shows a plot of errors obtained from solving the same problem using a conventional FEM code (using numerical linear algebra) for reference. We use a GMRES solver with a tolerance of $10^{-8}$.

\subsection{Poisson's equation with parametric log permeability}\label{sec:results-poisson-kl-problem}
Our second illustration is the solution of the PDE defined in \secref{sec:poisson-kl-problem-intro} (\eqnref{eq:poisson-kl-intro}), which is frequently used for simulating practical problems such as heat or mass transfer through an inhomogeneous media. We solve this problem in both 2D and 3D domains, i.e., $ D=[0,1]^2 $ as well as $ D=[0,1]^3 $; $ \tilde{\nu} $ is now a function of $ \mvec{x} = (x,y) $ and is also parametric, as mentioned in \eqnref{def:poisson-nu-harmonic-expansion}.

We seek a mapping of the form $ u = \neuralMap(\diffusivity) $, where $ \neuralMap $ denotes the neural network. So $ S^d $ refers to the discrete version of $ \diffusivity $. Following the principle of Karhunen-Loeve expansion as described in \ref{app:stochastic-representation-description}, the infinite dimensional random space is truncated into a finite dimensional space. Thereafter, a finite set of samples are prepared from this space for the training process. Suppose the number of samples is $ N_s $. Then the loss function takes the concrete form:
\begin{align}\label{def:loss-fem-poisson-kl-concrete}
J = \oneOver{N_s} \sum_{i = 1}^{N_s} \int_{\spatialD} \tilde{\nu}_i(\xvec)|\grad u_i(\xvec)|^2 d\mvec{x},
\end{align}
where, $ u_i = \neuralMap(\diffusivity_i, \netparams) $.
Both Dirichlet and Neumann conditions are present in this equation. In \neufenet{}, the Dirichlet conditions are applied exactly. The zero-Neumann condition is also applied exactly at the continuous level, since the boundary integrals vanish at the continuous level (see \eqnref{eq:def:bilnear-linear-form}).

\begin{figure}[t!]
	\centering
	\includegraphics[trim=120 0 50 0,clip,width=0.85\linewidth]{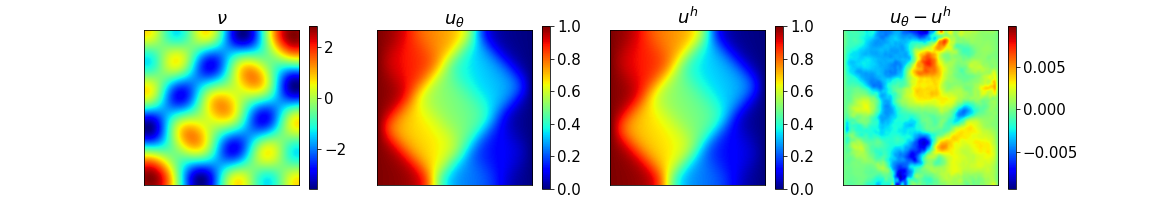}
	\caption{(Left) contours of $ \ln(\nu) $ corresponding to $ \underline{a} = (-0.24, -0.17, 0.13, 0.07, -0.07, 0.14) $, (middle-left) \neufenet{} prediction ($ \uNeural $), (midddle-right) reference numerical solution using FEM ($ \uDiscrete $), (right) contours of $ (\uNeural - \uDiscrete) $}
	\label{fig:poisson-kl-contours-2}
\end{figure}
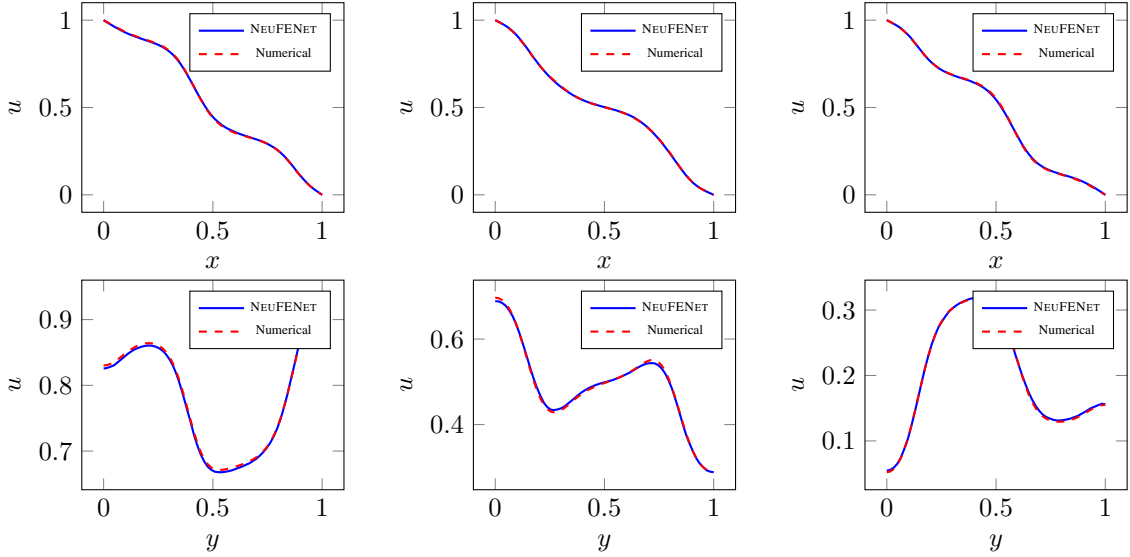
\begin{figure}[t]
	\centering
	\begin{minipage}{.31\textwidth}
		\centering
		\begin{tikzpicture}
		\begin{axis}[
		width=0.99\linewidth, 
		xlabel={$ x $}, 
		ylabel={$ u $},
		legend style={at={(0.95,0.95)},anchor=north east,legend columns=1}, 
		x tick label style={rotate=0,anchor=north} 
		]
		\addplot+ [thick,mark=none] table[x expr={\thisrow{s}},y expr={\thisrow{loc1}},col sep=comma]{results_data/kl-stochastic/anecdotal-query/ycuts_dfn.txt};
		\addplot+ [thick,mark=none, dashed] table[x expr={\thisrow{s}},y expr={\thisrow{loc1}},col sep=comma]{results_data/kl-stochastic/anecdotal-query/ycuts_fem.txt};
		\legend{\tiny{\neufenet{}}, \tiny{Numerical}}
		\end{axis}
		\end{tikzpicture}
	\end{minipage}
	\begin{minipage}{.31\textwidth}
		\centering
		\begin{tikzpicture}
		\begin{axis}[
		width=0.99\linewidth, 
		xlabel={$ x $}, 
		ylabel={$ u $},
		legend style={at={(0.95,0.95)},anchor=north east,legend columns=1}, 
		x tick label style={rotate=0,anchor=north} 
		]
		\addplot+ [thick,mark=none] table[x expr={\thisrow{s}},y expr={\thisrow{loc2}},col sep=comma]{results_data/kl-stochastic/anecdotal-query/ycuts_dfn.txt};
		\addplot+ [thick,mark=none, dashed] table[x expr={\thisrow{s}},y expr={\thisrow{loc2}},col sep=comma]{results_data/kl-stochastic/anecdotal-query/ycuts_fem.txt};
		\legend{\tiny{\neufenet{}}, \tiny{Numerical}}
		\end{axis}
		\end{tikzpicture}
	\end{minipage}
	\begin{minipage}{.31\textwidth}
		\centering
		\begin{tikzpicture}
		\begin{axis}[
		width=0.99\linewidth, 
		xlabel={$ x $}, 
		ylabel={$ u $},
		legend style={at={(0.95,0.95)},anchor=north east,legend columns=1}, 
		x tick label style={rotate=0,anchor=north} 
		]
		\addplot+ [thick,mark=none] table[x expr={\thisrow{s}},y expr={\thisrow{loc3}},col sep=comma]{results_data/kl-stochastic/anecdotal-query/ycuts_dfn.txt};
		\addplot+ [thick,mark=none, dashed] table[x expr={\thisrow{s}},y expr={\thisrow{loc3}},col sep=comma]{results_data/kl-stochastic/anecdotal-query/ycuts_fem.txt};
		\legend{\tiny{\neufenet{}}, \tiny{Numerical}}
		\end{axis}
	\end{tikzpicture}
	\end{minipage}
	\\
	\begin{minipage}{.31\textwidth}
		\centering
		\begin{tikzpicture}
		\begin{axis}[
		width=0.99\linewidth, 
		xlabel={$ y $}, 
		ylabel={$ u $},
		legend style={at={(0.95,0.95)},anchor=north east,legend columns=1}, 
		x tick label style={rotate=0,anchor=north} 
		]
		\addplot+ [thick,mark=none] table[x expr={\thisrow{s}},y expr={\thisrow{loc1}},col sep=comma]{results_data/kl-stochastic/anecdotal-query/xcuts_dfn.txt};
		\addplot+ [thick,mark=none, dashed] table[x expr={\thisrow{s}},y expr={\thisrow{loc1}},col sep=comma]{results_data/kl-stochastic/anecdotal-query/xcuts_fem.txt};
		\legend{\tiny{\neufenet{}}, \tiny{Numerical}}
		\end{axis}
		\end{tikzpicture}
	\end{minipage}
	\begin{minipage}{.31\textwidth}
		\centering
		\begin{tikzpicture}
		\begin{axis}[
		width=0.99\linewidth, 
		xlabel={$ y $}, 
		ylabel={$ u $},
		legend style={at={(0.95,0.95)},anchor=north east,legend columns=1}, 
		x tick label style={rotate=0,anchor=north} 
		]
		\addplot+ [thick,mark=none] table[x expr={\thisrow{s}},y expr={\thisrow{loc2}},col sep=comma]{results_data/kl-stochastic/anecdotal-query/xcuts_dfn.txt};
		\addplot+ [thick,mark=none, dashed] table[x expr={\thisrow{s}},y expr={\thisrow{loc2}},col sep=comma]{results_data/kl-stochastic/anecdotal-query/xcuts_fem.txt};
		\legend{\tiny{\neufenet{}}, \tiny{Numerical}}
		\end{axis}
		\end{tikzpicture}
	\end{minipage}
	\begin{minipage}{.31\textwidth}
		\centering
		\begin{tikzpicture}
		\begin{axis}[
		width=0.99\linewidth, 
		xlabel={$ y $}, 
		ylabel={$ u $},
		legend style={at={(0.95,0.95)},anchor=north east,legend columns=1}, 
		x tick label style={rotate=0,anchor=north} 
		]
		\addplot+ [thick,mark=none] table[x expr={\thisrow{s}},y expr={\thisrow{loc3}},col sep=comma]{results_data/kl-stochastic/anecdotal-query/xcuts_dfn.txt};
		\addplot+ [thick,mark=none, dashed] table[x expr={\thisrow{s}},y expr={\thisrow{loc3}},col sep=comma]{results_data/kl-stochastic/anecdotal-query/xcuts_fem.txt};
		\legend{\tiny{\neufenet{}}, \tiny{Numerical}}
		\end{axis}
		\end{tikzpicture}
	\end{minipage}
	\caption{Line cuts for the contours presented in Fig \ref{fig:poisson-kl-contours-2}. (Top row) $ x $-parallel line cuts at $ y = $ 0.2, 0.5, 0.8. (Bottom row) $ y $-parallel line cuts at $ x = $ 0.2, 0.5, 0.8.}
	\label{fig:line-cuts-kl-2d-single}
\end{figure}

By optimizing the loss (\eqnref{def:loss-fem-poisson-kl-concrete}), we attempt to  learn the distribution of the stochastic solution, given that the coefficients in the log permeability K-L sum come from a known range of values that depends on the parameter space $ \omega $. We truncate the K-L sum after 6 terms. These six coefficients form a six-dimensional space from which the coefficient tuples $ \{a_i\}_{i=1}^{6} $ can be drawn. The \neufenet{} is trained by selecting a finite number ($ N_s $) of pseudo-random samples from this 6-dimensional space, specifically $ \mvec{a}\in [-\sqrt{3}, \sqrt{3}]^6 $ (see also \ref{app:stochastic-representation-description}). For the results shown below, we have taken $ N_s=65536 $. We used the Adam optimization algorithm, with a learning rate of $ 10^{-4} $. Once the network is trained, \textit{\textbf{we can perform inference by evaluating the solution for any diffusivity $ \diffusivity $ taken from the sample space}}. To illustrate the nature of the input ($ \diffusivity $ or $ S^d $) and the solution ($ u $ or $ U^d $), we present an anecdotal (i.e., a non-special and random) set of $ S^d $ and $ U^d $ in \figref{fig:poisson-kl-contours-2}. A reference solution using a conventional FEM program is also presented therein. Furthermore, sectional line cuts for these contours are shown in \figref{fig:line-cuts-kl-2d-single}. The line cuts display a close match between \neufenet{} and the numerical solution.

\input{figure_panels/stochastic_histogram_panel-unet.tex}
\subsubsection{Statistical distribution of solution}\label{sec:statistical-dist}

Since \eqnref{eq:poisson-kl-intro} is parametric, we can compare the quality of solution from \neufenet{} with a reference numerical solution in a statistical manner. We choose some points on the domain $ \spatialD $ (shown by black dots in \figref{fig:query-points-for-statistics}) and evaluate the solution values at those query points for $ 16,384 $ samples of $ \diffusivity $, where the $ \{a_i\}_{i=1}^{6} $ tuples are taken from a smaller subset of the full sample space, namely, $ \mvec{a}\in [-\oneOver{4},\oneOver{4}]^6 $. Therefore , at a particular query point, those $ 16,384 $ solution values approximately represent a distribution of the solution values at that point. This sample of solution values at each of those query points then allow us to create histograms of the solution values at each query point.

The histograms are shown in \figref{fig:stochastic-histograms}. We notice a very close match between the histograms obtained from \neufenet{} and a conventional FEM solver. This once again confirms that \neufenet{} is effective in providing the correct statistics of the parametric Poisson equation.

\subsection{3-D Poisson's equation}\label{sec:results-poisson-kl-problem-3d}
So far, we have focused on two-dimensional problems. Although 2D problems are useful in demonstrating the key features and properties of \neufenet{} method, a real test of neural PDE solvers lies in their ability to solve three-dimensional problems. \neufenet{} can solve both 2D and 3D problems without much changes to the architecture. To provide an example of this, we solve the 3D counterpart of the PDE defined in \eqnref{eq:poisson-kl-intro}, which is:
\begin{subequations}\label{eq:poisson-kl-3d}
	\begin{align}
	-\grad\cdot(\tilde{\nu}(\mvec{x})\grad u) &= 0 \text{  in  } D=[0,1]^3, \\
	u(x=0, y,z) &= 1, \\
	u(x=1, y,z) &= 0, \\
	\hat{\mvec{n}}\cdot\grad u &= 0 \quad \text{on all other boundaries},
	\end{align}
\end{subequations}
where $ \hat{\mvec{n}} $ denotes the outward normal to the boundary. Once again, the functional form of $ \tilde{\diffusivity(\mvec{x})} $ is described in \ref{app:stochastic-representation-description}. The loss function is a direct analogue of Equation \ref{def:loss-fem-poisson-kl-concrete}:
\begin{align}\label{def:loss-fem-poisson-kl-concrete-3d}
J = \int_{\spatialD} \tilde{\nu}(\xvec)|\grad u(\xvec)|^2 d\mvec{x},
\end{align}
In \figref{fig:poisson-kl-3d-contours}, we show one randomly selected pair of $ \diffusivity $ and $ u $ for a 3D problem obtained using \neufenet{}. We optimize \eqnref{def:loss-fem-poisson-kl-concrete-3d} for a randomly selected set of coefficients $ \mvec{a} = (-1,1.4,1.5,-1.3,-1.6,0.3) $. The plots of $ \diffusivity(x,y,z) $ and $ \uNeural(x,y,z) $ are shown in \figref{fig:poisson-kl-3d-contours}.

\begin{figure*}[!h]
	\centering
	\includegraphics[trim=300 105 40 170,clip,width=0.45\linewidth]{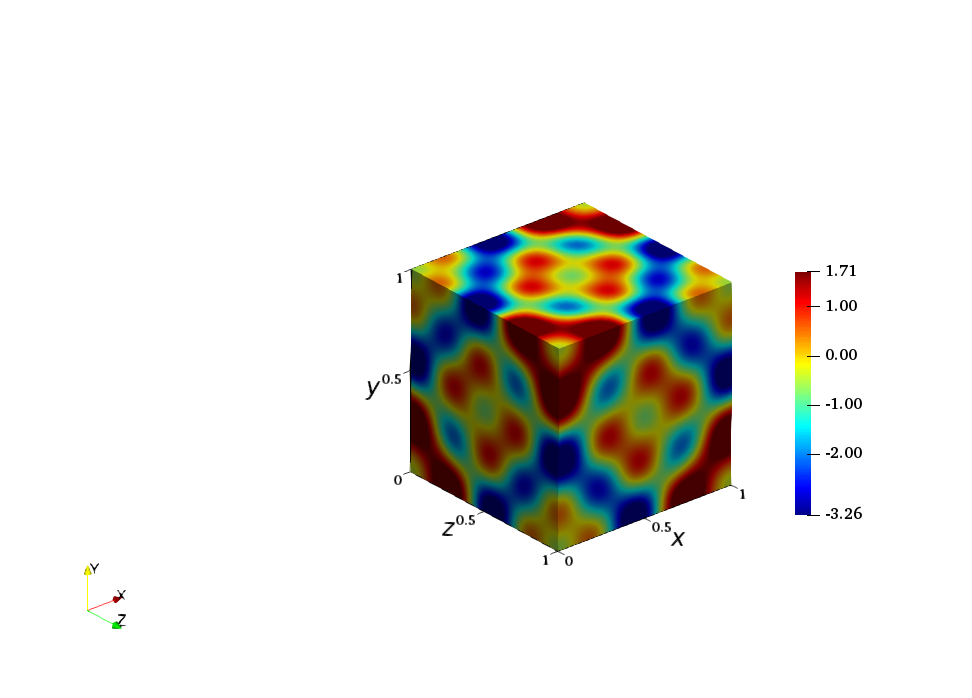}
	\includegraphics[trim=300 105 40 170,clip,width=0.45\linewidth]{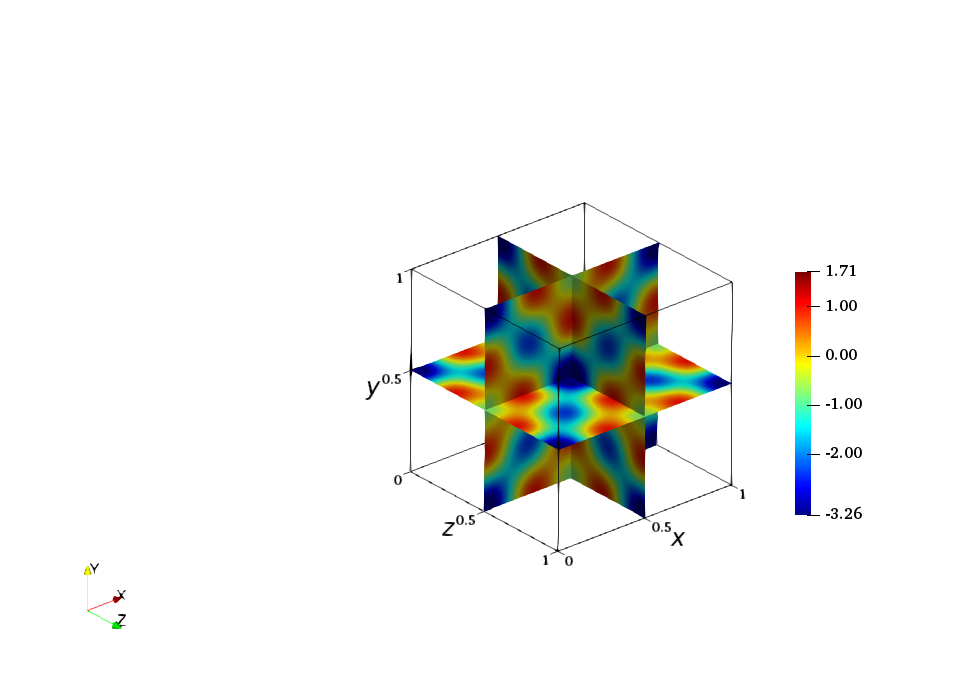}
	\includegraphics[trim=300 105 40 170,clip,width=0.45\linewidth]{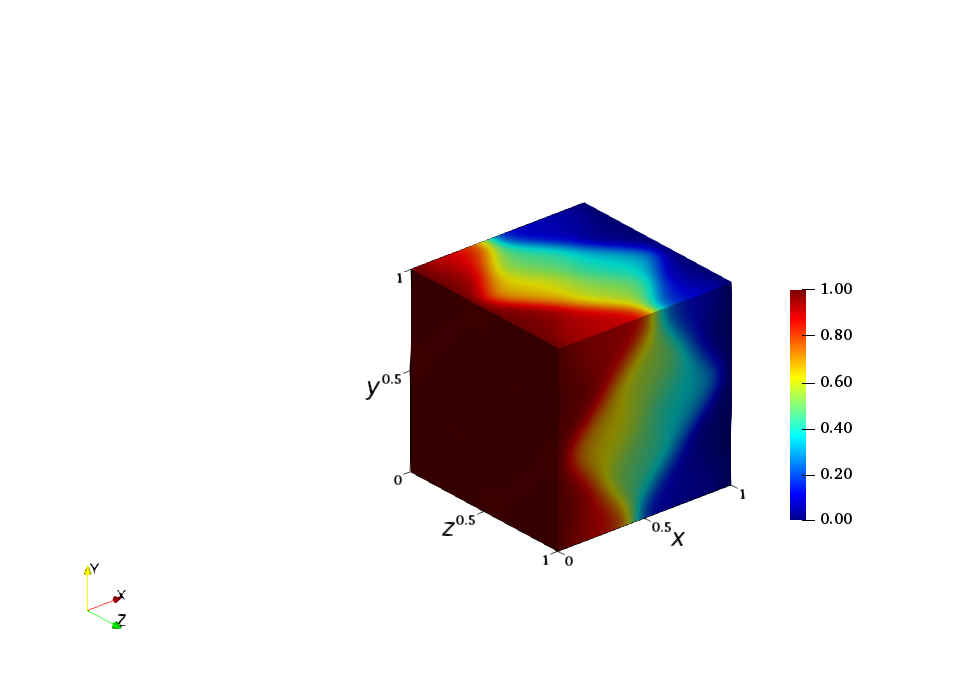}
	\includegraphics[trim=300 105 40 170,clip,width=0.45\linewidth]{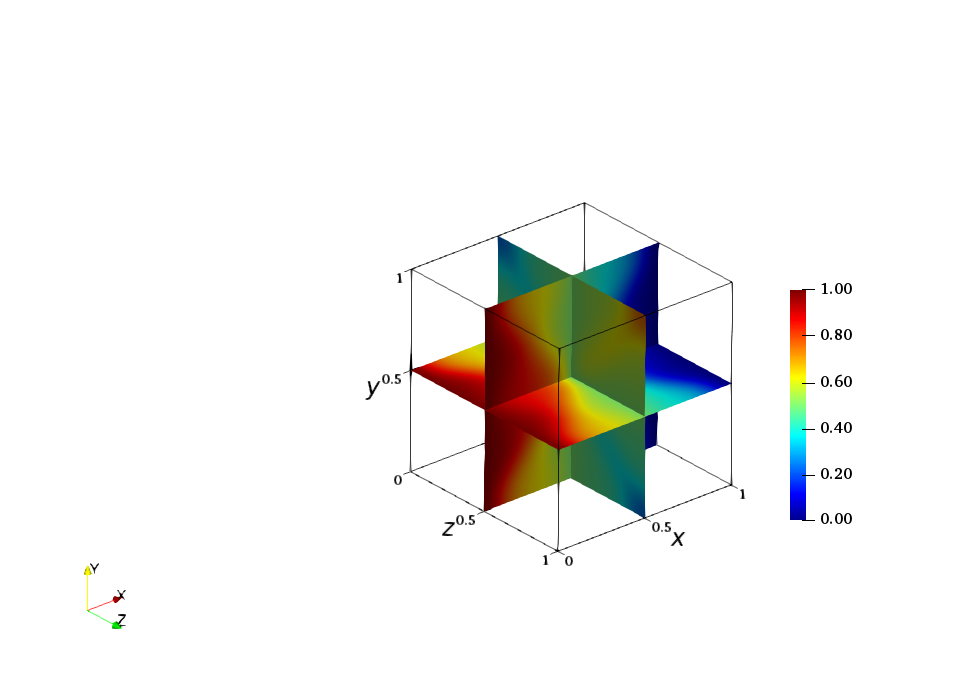}
	\caption{Contours of $\ln(\nu(x,y,z))$ and the solution $\uNeural(x,y,z)$ to the 3D Poisson's problem (\eqnref{eq:poisson-kl-3d}) on a $ 64\times 64\times 64 $ mesh (for $ \mvec{a} = (-1,1.4,1.5,-1.3,-1.6,0.3) $).}
	\label{fig:poisson-kl-3d-contours}
\end{figure*}

\section{Conclusions and future directions}

In this paper we develop a neural method \neufenet{} for solving parametric PDEs where the discretization and the loss functions are inspired by the continuous Galerkin (cG) method and the Rayleigh-Ritz method respectively. Due to the choice of discretization scheme, \neufenet{} inherits the approximation properties of the cG method. This allows us to: (i) calculate spatial derivatives in the same way as in finite element methods, (ii) perform spatial integration using simple Gaussian quadrature schemes, (iii) apply Dirichlet and (zero) Neumann boundary conditions exactly and (iv) derive \textit{a priori} error estimates.

The optimization problem is defined in terms of an energy functional derived from variational principles. This is in contrast to residual based minimization which is the more widely followed process across the current neural methods, albeit with exceptions (such as \citet{yu2017deep}). We showed examples of Poisson's equation solved using \neufenet{}. Since Poisson's equation is a self-adjoint equation, its energy functional is convex and thus possesses a unique minima, which can be easily found by a gradient based optimization method. We further illustrate that such a method can successfully be used to solve stochastic PDEs and determine its statistical properties.

We identify some disadvantages and opportunities for future work. \neufenet{} as described in this paper suits ``steady-state'' equations such as the Poisson's equation or other similar equations. We hypothesize that \neufenet{} may require additional features for solving parabolic equations because the full space-time functional is not convex. Since \neufenet{} uses discretization provided by a finite element method, therefore \neufenet{} is mesh based. It must be noted that for the examples shown in this paper, we did not need to store a mesh explicitly because the meshes were structured meshes and fully regular. But for an arbitrary geometry, such a mesh will need to be saved to memory. Because of the requirement of mesh in \neufenet{}, it can be memory-intensive, especially for 3D or higher dimensional problems. But this issue can be alleviated by considering distributed frameworks (such as \citet{balu2021distributed}). We anticipate that such approaches that tightly integrate neural architectures with well developed scientific computing approaches will prove successful towards our goal of a (near) real time neural PDE inference.

\section*{Acknowledgements}
This work was supported in part by the National Science Foundation under grants CCF-2005804, LEAP-HI-2053760, CMMI-1644441, CPS-FRONTIER-1954556, USDA-NIFA-2021-67021-35329 and ARPA-E DIFFERENTIATE-DE-AR0001215. Any information provided and opinions expressed in this material are those of the author(s) and do not necessarily reflect the views of, nor any endorsements by, the funding agencies.



\bibliographystyle{elsart-num-names}
\bibliography{reflist}

\clearpage
\appendix
\appendix
\section{Representation of random diffusivity}
\label{app:stochastic-representation-description}
With $ \omega $ taken from the sample space $ \Omega $, the diffusivity / permeability $ \diffusivity $  can be written as an exponential of a random quantity $ Z $:
\begin{align}\label{def:log-perm}
\diffusivity = \exp \left( Z(\xvec; \omega) \right).
\end{align}
We assume that $ Z $ is square integrable, i.e., $ \mathbb{E}\left[ |Z(\xvec; \omega)|^2 \right] < \infty $. Then we can write $ Z $ using the Karhunen-Loeve expansion \cite{ghanem2003stochastic}, as:
\begin{align}
Z(\xvec; \omega) = \bar{Z}(\xvec) + \sum_{i = 1}^{\infty} \sqrt{\lambda_i}\phi_i(\xvec)\psi_i(\omega),
\end{align}
where $\bar{Z}(\xvec) = \mathbb{E}(Z(\xvec, \omega)) $, and $ \psi_i(\omega) $ are independent random variables with zero mean and unit variance. $ \lambda_i $ and $ \phi_i(\xvec) $ are the eigenvalues and eigenvectors corresponding to the Fredholm equation:
\begin{align}\label{eq:fredholm-integral}
\int_D C_Z(\mvec{s},\mvec{t})\phi(s)ds = \lambda\phi(\mvec{t}),
\end{align}
where $ C_Z(s,t) $ is the covariance kernel given by,
\begin{align}
C_Z(\mvec{s}, \mvec{t}) = \sigma_z^2 \exp \left( -\left[ \frac{s_1-t_1}{\eta_1}+\frac{s_2-t_2}{\eta_2}+\frac{s_3-t_3}{\eta_3} \right] \right),
\end{align}
where $ \eta_i $ is the correlation length in the $ x_i $ coordinate. This particular form of the covariance kernel is separable in the three coordinates, thus the eigenvalues and the eigenfunctions of the multi-dimensional case can be obtained by combining the eigenvalues and eigenfunctions of the one-dimensional covariance kernel given by:
\begin{align}\label{eq:covariance-1d}
C_Z(s,t) = \sigma_Z^2\exp \left(-\frac{s-t}{\eta}\right),
\end{align}
where $ \sigma_Z $ is the variance and $ \eta $ is the correlation length in one-dimension. 

\eqnref{def:log-perm} can then be written as,
\begin{align}\label{def:poisson-nu-harmonic-expansion}
\tilde{\nu}(\mvec{x}; \omega) 
& = \exp\left(\sum_{i = 1}^{m} a_i \sqrt{\lambda_{xi}\lambda_{yi}} \phi_i(x)\psi_i(y) \right)
\end{align}
where $a_i$ is an $m$-dimensional parameter, $ \lambda_x $ and $ \lambda_y $ are vectors of real numbers arranged in the order of monotonically decreasing values; and $ \phi $ and $ \psi $ are functions of $ x $ and $ y $ respectively. $ \lambda_{xi} $ is calculated as:
\begin{align}
	\lambda_{xi} = \frac{2\eta\sigma_x}{(1+\eta^2 \omega_x^2)},
\end{align}
where $ \omega_x $ is the solution to the system of transcendental equations obtained after differentiating \eqnref{eq:fredholm-integral} with respect to $ \mvec{t} $. $ \lambda_{yi} $ are calculated similarly. $ \phi_i(x) $ are given by:
\begin{align}
	\phi_i(x) = \frac{a_i}{2}\cos(a_i x) + \sin(a_i x)
\end{align}
and $ \psi_i(y) $ are calculated similarly. We take $ m = 6 $ and assume that each $ a_i $ is uniformly distributed in $ [-\sqrt{3},\sqrt{3}] $, thus $ \mvec{a} \in [-\sqrt{3},\sqrt{3}]^6 $. The input diffusivity $ \diffusivity $ in all the examples in \cref{Sections}{sec:results-poisson-kl-problem} and \ref{sec:results-poisson-kl-problem-3d} are calculate by choosing the 6-dimensional coefficient $ \mvec{a} $ from $ [-\sqrt{3},\sqrt{3}]^6 $.

\clearpage
\section{Further discussion on convergence studies}\label{app:error-convergence}
\subsection{Discussion on the role of keeping \texorpdfstring{$\eCap$ and $ \eOpt $}{} low}
If we choose a fixed network architecture and use it to solve \eqnref{eq:poisson-manufactured-pde} across different $ h $-levels, then the errors do not necessarily decrease with decreasing $ h $. As shown in \figref{fig:convergence-plot-non-parametric-clone}, the errors actually increase when $ h>2^{-5} $. This reason for this behaviour is that, when $ h $ becomes low, the number of discrete unknowns in the mesh (i.e., $ U_i $'s in \eqnref{eq:def:fem-function-approximation-0}) increases. In fact, in this case, the number of basis functions / unknowns, $ N $  is exactly equal to $ \oneOver{h^2} $. As $ h $ decreases, the size of the space $ \spaceVh $ increases. But since the network remains the same, the discrete function space $ \spaceVh $ does not remain a subspace of $ \spaceVhtheta $ anymore. This network function class also needs to get bigger to accommodate all the possible functions at the lower values of $ h $. \figref{fig:convergence-plot-non-parametric-clone} also shows the errors obtained when the network is indeed enhanced to make $ \spaceVhtheta\supset\spaceV $ (this is a clone of the errors plotted in \figref{fig:convergence-plot-non-parametric}).
\begin{figure}[!h]
	\centering
	\begin{tikzpicture}
	\begin{loglogaxis}[
	width=0.5\linewidth, 
	height=0.5\linewidth,
	xlabel=$h$, 
	ylabel=$\norm{\uNeural - u}_{L^2(\spatialD)}$,
	legend style={at={(0,1.02)},anchor=south west,legend columns=1}, 
	x tick label style={rotate=0,anchor=north}, 
	xtick={0.00390625, 0.0078125, 0.015625 , 0.03125  , 0.0625, 0.125},
	xticklabels={$2^{-8}$,$2^{-7}$,$2^{-6}$,$2^{-5}$,$2^{-4}$,$2^{-3}$},
	legend cell align={left},
	]
	\addplot+[thick,color=red,mark options={fill=red}]
	table[x expr={1.0/((\thisrow{elm}))},y expr={\thisrow{diffnet_with_network_capacity_saturated}},col sep=comma]{results_data/poisson-convergence/errors_linear.txt};
	\addplot+ [dashed,color=black,mark options={fill=black}]
	table[x expr={1.0/((\thisrow{elm}))},y expr={\thisrow{diffnet_with_network_capacity_enhanced}},col sep=comma]{results_data/poisson-convergence/errors_linear.txt};
	\logLogSlopeTriangle{0.8}{0.2}{0.2}{2}{blue};
	\legend{\footnotesize \neufenet{} + Network (capacity unchanged), \footnotesize \neufenet{} + Network (capacity enhanced)}
	\end{loglogaxis}
	\end{tikzpicture}
	\caption{Convergence of the error in $L^2$ norm for the Poissson equation with analytical solution $u =\sin{(\pi x)}\sin{(\pi y}$).}
	\label{fig:convergence-plot-non-parametric-clone}
\end{figure}

\clearpage
\section{Solutions to the parametric Poisson's equation}
\subsection{Randomly selected examples}
\begin{table}[!htb]
	\centering
	\small
    \setlength\extrarowheight{4pt}
	\csvreader[
	tabular=r|rrrrrr|rrrr,
	table head=\hline \# &{$ a_1 $} & {$ a_2 $} &{$ a_3 $} & {$ a_4 $} & {$ a_5 $} & {$ a_6 $} & {$\normL[2]{u_{\netparams}}{D}$} & {$\normL[2]{u^h}{D}$} & {$\normL[2]{u_{\netparams}-u^h}{D}$} & {$\frac{\normL[2]{u_{\netparams}-u^h}{D}}{\normL[2]{u^h}{D}} $} \\ 
	\hline,
	late after last line=\\ \hline 
	]{
		results_data/kl-stochastic/anecdotal-query/norms.txt
	}{}{\csvlinetotablerow}
	\caption{Norm of solution fields for a few randomly selected }
	\label{tab:kl-stochastic-anecdotal-contours}
\end{table}
\begin{figure}[!htb]
	\centering
	\includegraphics[trim=140 150 110 170,clip,width=0.95\linewidth]{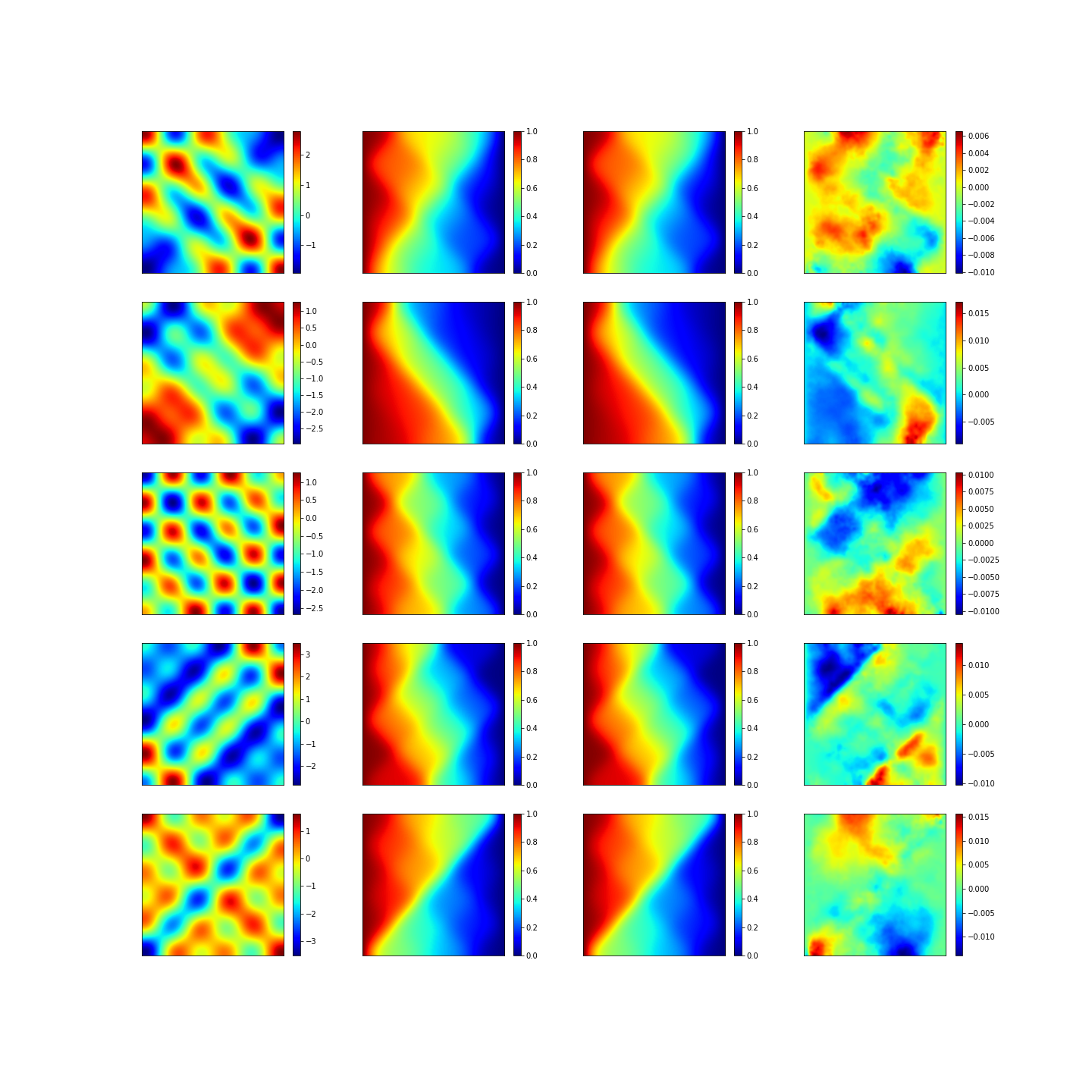}
	\caption{Contours for the randomly selected examples presented in \tabref{tab:kl-stochastic-anecdotal-contours}: (left) $ \ln(\diffusivity) $, (mid-left) $ \uNeural $, (mid-right) $ \uDiscrete $ and (right) $ (\uNeural-\uDiscrete) $}
	\label{fig:poisson-kl-2d-anecdotal-contours}
\end{figure}
\clearpage

\subsection{Mean and standard-deviation fields}
\begin{table}[!htb]
	\centering
    \setlength\extrarowheight{4pt}
	\csvreader[
	tabular=ccccc,
	table head=\hline & {$ \norm{\uNeural}_{L^2(\spatialD)} $} & {$ \norm{\uDiscrete}_{L^2(\spatialD)} $} & {$ \norm{\uNeural-\uDiscrete}_{L^2(\spatialD)} $} & { $ \frac{\norm{\uNeural-\uDiscrete}_{L^2(\spatialD)}}{\norm{\uDiscrete}_{L^2(\spatialD)}} $} \\ 
	\hline,
	late after last line=\\ \hline 
	]{
		results_data/kl-stochastic/mean-sdev-norms.txt
	}{}{\csvlinetotablerow}
	\caption{Norm of the mean and standard-deviation fields }
	\label{tab:poisson-kl-stochastic-mean-sdev}
\end{table}
\begin{figure}[!htb]
	\centering
	\begin{minipage}{.8\textwidth}
		\centering
		{\includegraphics[trim=100 0 50 0, clip,width=0.9\linewidth]{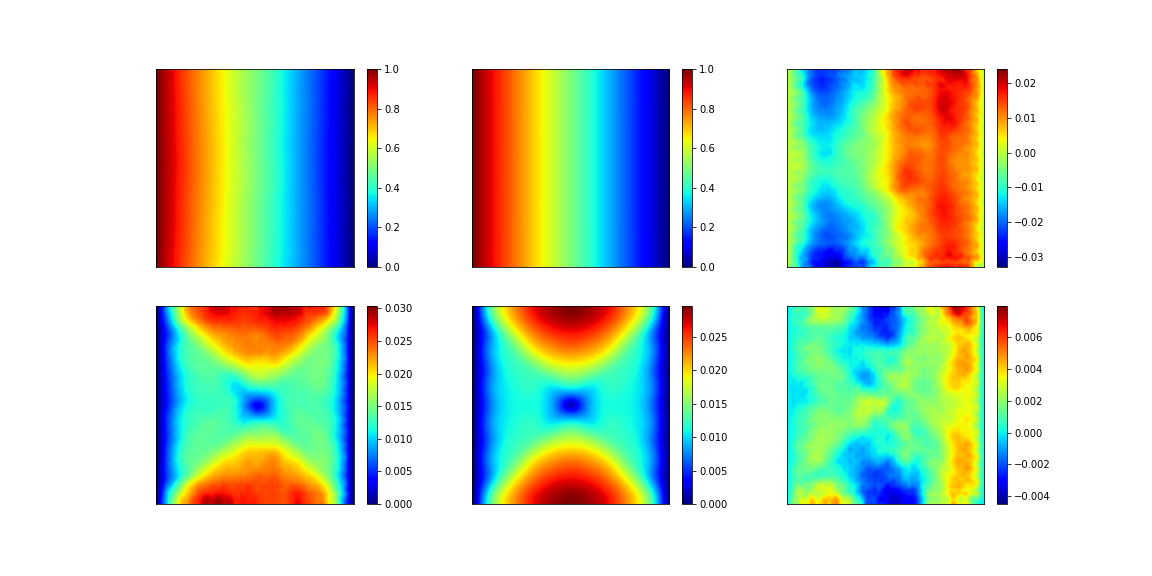}}
	\end{minipage}	
	\caption{(top) Mean  and (bottom) standard deviation fields: (left) \neufenet{} ($ \uNeural $), (mid) Conventional FEM ($ \uDiscrete $, (right) point-wise difference ($ \uNeural-\uDiscrete $)}
	\label{fig:poisson-kl-stochastic-mean-sdev}
\end{figure}

\clearpage

\end{document}